\documentclass{article}

\usepackage[utf8]{inputenc}
\usepackage[margin=1in]{geometry}
\usepackage{microtype}
\usepackage{graphicx, wrapfig}
\usepackage{booktabs}
\usepackage{hyperref}

\title{Testing Tail Weight of a Distribution Via Hazard Rate}

\usepackage{soul, amsmath}
\usepackage{amsfonts, bbm}
\usepackage{graphicx}
\usepackage{subcaption}
\usepackage{amsthm}
\usepackage{thm-restate}
\usepackage{algorithm, algpseudocode}
\usepackage{todonotes}
\usepackage{mdframed}
\usepackage{framed}
\usepackage{wrapfig}
\usepackage{chngcntr}
\usepackage{enumitem}
\usepackage[utf8]{inputenc} 
\usepackage[T1]{fontenc}    
\usepackage{hyperref}       
\usepackage{url}            
\usepackage{booktabs}       
\usepackage{amsfonts}       
\usepackage{nicefrac}       
\usepackage{microtype}      
\usepackage{cleveref}

\usepackage{tikz}
\usepackage{mathdots}
\usepackage{yhmath}
\usepackage{cancel}
\usepackage{siunitx}
\usepackage{array}
\usepackage{multirow}
\usepackage{amssymb}
\usepackage{textcomp}
\usepackage{gensymb}
\usepackage{tabularx}
\usepackage{booktabs}
\usetikzlibrary{fadings}
\usetikzlibrary{patterns}
\usetikzlibrary{shadows.blur}
\usetikzlibrary{shapes}

\newtheorem{theorem}{Theorem}[section]
\newtheorem{definition}{Definition}[section]
\newtheorem{corollary}{Corollary}[section]

\newtheorem{claim}[theorem]{Claim}

\DeclareMathOperator{\erf}{erf}

\author{
 Maryam Aliakbarpour \thanks{Boston University / Northeastern University; NSF grants CNS-2120667, CNS-2120603, CCF 1934846, and BU's Hariri Institute for Computing; Email: \href{mailt:omaryama@alum.mit.edu}{maryama@alum.mit.edu}}, 
 Amartya Shankha Biswas \thanks{Massachusetts Institute of Technology, supported by NSF grants  CCF-2006664, IIS-1741137, CCF-1733808, and Fintech@CSAIL; Email: \url{amartyashankha@gmail.com}}, 
 Kavya Ravichandran \thanks{Toyota Technological Institute at Chicago; Work predominantly done while author was a student at Massachusetts Institute of Technology; Email: \url{ravichandran.kavya@gmail.com}}, 
 Ronitt Rubinfeld \thanks{Massachusetts Institute of Technology, supported by NSF grants CCF-2006664, DMS-2022448, IIS-1741137, CCF-1733808, CCF-1740751, CCF-2006664, and Fintech@CSAIL; Email: \url{ronitt@csail.mit.edu}}
 }

\date{\today}

\usepackage{mathtools}

\hypersetup{
    colorlinks,
    linkcolor={purple},
    citecolor={teal},
    urlcolor={violet}
}

\setlist[enumerate]{label={\arabic*.}}

\newcommand\E[2][]{
\ensuremath{\mathrm{\mathbf{E}}_{#1} 
\pmb{\left[\vphantom{#2}\right.}
{#2}
\pmb{\left.\vphantom{#2}\right]}
}}

\renewcommand\Pr[2][]{
\ensuremath{\mathrm{\mathbf{Pr}}_{#1} 
\pmb{\left[\vphantom{#2}\right.}
{#2}
\pmb{\left.\vphantom{#2}\right]}
}}

\def\AA{\mathcal{A}}

\def\CC{\mathcal{C}}
\def\DD{\mathcal{D}}
\def\EE{\mathcal{E}}

\def\fbar{\ensuremath{\overline{f}_H}}
\def\poi{\ensuremath{\mbox{\sl Poi}}}

\def\fE{\ensuremath{f_{\text{exp}}}}
\def\fH{\ensuremath{{f_H}}}
\def\FE{F_{\text{exp}}}
\def\FH{F_H}

\begin{document}

\maketitle

\begin{abstract}
Understanding the shape of a distribution of data is of interest to people in a great variety of fields, as it may affect the types of algorithms used for that data. We study one such problem in the framework of {\em distribution property testing}, characterizing the number of samples required to to distinguish whether a distribution has a certain property or is far from having that property. In particular, given samples from a distribution, we seek to characterize the tail of the distribution, that is, understand how many elements appear infrequently. We develop an algorithm based on a careful bucketing scheme that distinguishes light-tailed distributions from non-light-tailed ones with respect to a definition based on the hazard rate, under natural smoothness and ordering assumptions. We bound the number of samples required for this test to succeed with high probability in terms of the parameters of the problem, showing that it is polynomial in these parameters. Further, we prove a hardness result that implies that this problem cannot be solved without any assumptions.
\end{abstract}

\section{Introduction} \label{ch:intro}

Testing properties of a data distribution is a fundamental problem with applications in many scientific endeavors. The goal of distribution testing is to efficiently discern whether observed data confirm a hypothesized model or not. Such problems have been studied in asymptotic statistics for over a century (\cite{pearson00, np33}), where the aim is to provide tests with vanishing error as the number of samples goes to infinity (i.e., asymptotically). In the past two decades, the field of {\em property testing of distributions} has sought to characterize what error can be achieved with finitely many samples (i.e., non-asymptotic setting) (\cite{batu_testing_2000, batu_testing_2013}).
The objective of the problem here is to distinguish whether a distribution has some property (null hypothesis) or is far from having that property (alternative hypothesis) with high constant probability using as few samples as possible. Finite-sample guarantees 
have been given for testing a wide range of distribution properties including uniformity, monotonicity, low-error representation by a $k$-histogram function, support size estimation, and many others (\cite{batu_sublinear_2004}, \cite{cj19-01}, \cite{canonne_k_histogram_2016}, \cite{canonne_survey_2020}, \cite{aakmrx_kwise_indep_07}, \cite{ad_pbd_2015}, \cite{cdvv_closeness_2014}, \cite{valiant_symmetric_2011}), \cite{rubinfeld_monotone_2020}, \cite{aliakbarpourGPRY18}
.\par

Much of the work in distribution property testing makes no assumptions on the distributions being tested and focuses on distributions over discrete domains. However, strong lower bounds have been given for many of these problems, involving dependence on the size of the support of the distribution. Thus, for continuous domains, some assumptions on the distributions are required for the problem to be tractable (e.g., \cite{monotone_cont_highdim_2010}). We focus on this latter case. 

We are interested in the shape of the ``tail'' of a distribution, that is, the behavior of the distribution as it moves away from the mean.
In some distributions, the frequencies of elements far from the mean drop very quickly (``light-tailed''), while in others, the frequencies of large elements drop more slowly (``heavy-tailed'').
\cite{harchol-balter_e_nodate} has shown that the performances of policies for scheduling 
computing jobs vary dramatically depending on whether the 
distribution of the job workloads are heavy-tailed or light-tailed.
The shape of the distribution has influenced the design and analysis of learning algorithms --
e.g.,
the classification algorithm of~\cite{wang_learning_2017}, the generalization bound of~\cite{feldman_does_2019}, and the frequency estimation result of~\cite{hsu_learning-based_2018}
(where the latter two are for specific
subclasses of heavy-tailed distributions).
In this work, we seek to characterize the tail of the distribution, specifically to decide whether it is substantially ``heavy.'' 

\paragraph{``Heavy-Tailed'' distributions}
It is non-trivial to give a single unifying definition of heavy-tailed distributions, 
since the definition needs to accommodate a wide range of behaviors, including distributions whose tails fall off at irregular rates. 
We discuss the plethora of definitions found in the literature in \Cref{appendix:defns_ht}. Though not equivalent, these definitions are united by the fact that the point of reference is the exponential distribution, meaning that a distribution whose tail decays more slowly than the exponential is considered ``heavy-tailed.''~\footnote{Distributions with regularly varying tails represent a subset of distributions whose tails decay more slowly than exponential.}
In this work, we adapt Klugman et al.'s definition of heavy-tailed that is based on the property that the 
hazard rate of a heavy-tailed distribution is 
decreasing~\cite{klugman_loss_2004}; the characterization is of similar structure to other definitions,
admits a clean description, and reflects the idea that heavy tails decay more slowly than exponential tails. 


\paragraph{Our Setting} 
We consider distributions over continuous and unbounded domains in the non-parametric setting, that is, not assuming that distributions belong to any specific class. We access such distributions via independent and identically-distributed samples. \par

In order to make the problem tractable, two kinds of technical assumptions are essential: smoothness and monotonicity. The first type of assumption, \textit{smoothness}, is typical in learning theory and non-parametric statistics. Smoothness limits the behavior of the characteristic function of a distribution,
protecting against adversarial behavior on small regions of a continuous domain not seen by a finite set of samples: without such assumptions, for any finite number of samples, one could construct two distributions that look the ``same" when we draw a finite set of samples but differ on tiny intervals of the domain that are not detected by those samples. The second kind of assumption we require arises from the fact that distribution having a tail implicitly relies on the distribution decaying. Namely, we assume that the distributions we consider are monotone decreasing (or more broadly, unimodal). These conditions are discussed in detail in Section~\ref{ch:prelim}. 


\paragraph{Our Contributions} 
In this work, we begin by giving a parametrized  definition of heavy-tailed distributions,
based on an extension of the hazard rate definition from \cite{klugman_loss_2004}. \footnote{The hazard rate of a distribution at a given point in the support is the value of the PDF divided by the amount of mass left in the tail from that point.}
In our definition, a distribution is {\em light-tailed} if it has non-decreasing hazard rate throughout the domain. On the other hand, a heavy-tailed distribution must show a behavior ``far'' from light-tailed distributions at least in some interval of the domain: 

\begin{definition}[informal statement] A distribution is called $(\alpha, \rho)$-{\em Heavy-tailed} if the hazard rate decreases by at least rate $\alpha$ on a contiguous portion of the domain that contains at least $\rho$ of the probability mass. If the hazard rate is non-decreasing, the distribution is called {\em light-tailed}.
\end{definition}
 For the formal definition, see Definition~\ref{def:alpha-Heavy-Tailed}. This parametrization allows for a fine-grained characterization of tail shape and
 might allow for more nuanced algorithm design. Further, we give a hardness result that shows that we cannot solve the problem in this domain without structural assumptions (Section~\ref{sec:lowerbound}), justifying the need for a mild condition on the contiguousness of heavy-tailed regions.


\par
With this definition in mind, we seek to design an efficient algorithm that, given finitely many samples from a distribution, determines whether they come from a light-tailed or $(\alpha, \rho)$-Heavy-tailed distribution.
The main result of our work is a theorem stating the number of samples (up to constant factors) that suffice 
to perform this task, presented here informally:
\begin{theorem}[informal statement] We can distinguish between light-tailed distributions and $(\alpha, \rho)$-heavy-tailed ones with a number of samples that depends polynomially on smoothness parameters, $\alpha,$ and $\rho$.
\end{theorem}

Finally, we run experiments on synthetic and real-world data to show the feasibility of our algorithm in practice. We also show that our algorithm outperforms a naive one in its ability to detect a subtly heavy-tailed distribution. Synthetic data experiments can be found in Section~\ref{sec:experiments} and the rest in Appendix~\ref{appendix:discussion_of_experiments} and Appendix~\ref{appendix:sec-low-sample}.

\paragraph{Our Approach} Our algorithm is based on a simple but fundamental observation about the rate of dropoff as it relates to the weight of a distribution's tail: in the tail of a heavy-tailed distribution, we would expect that the distance in the support required to accumulate a fixed amount of weight would not change too much since it drops more slowly, whereas in a light tailed distribution, this quantity is much more drastic. Based on this, we introduce a  \emph{proxy quantity} in Section~\ref{sec:proxy_exact_test} (which we will call $S$) that measures 
how long it takes the distribution to accumulate some amount of weight relative to how long it took previously, acting as a proxy for calculating the derivative of the hazard rate.

The main challenge here is that we cannot compute $S$ from the samples directly since it is a function of the density function of the distribution. Therefore, we present a test statistic, called $\hat S$, to approximate $S$ from samples. The algorithm makes use of a bucketing scheme that partitions the domain of the distribution into buckets (intervals) that contain equal probability mass. The algorithm then uses the lengths of these intervals to calculate $\hat S$ and compares it to a threshold that separates light-tailed distributions from heavy-tailed ones. 

We determine the sample complexity for the algorithm in the defined setting by showing that if we draw ``enough'' samples, $\hat S$ is an accurate estimate of $S$  (Theorem~\ref{thm:main}) and lies on the correct side of the threshold for any underlying distribution. In calculating $\hat S$, we incur two main sources of error: first, we approximate derivatives involved in computation of $S$ by the discrete derivative/difference quotient; second, our algorithm uses order statistics from the samples to define the buckets, which introduces error in the estimation of the lengths of the buckets. Our assumption about smoothness allows us to analyze the former, and for the latter, we evaluate the concentration of order statistics.

It is worth noting that precise estimation of bucket endpoints from samples is challenging, since a small amount of probability mass could lie in a very large interval. Indeed, this is the most technically-interesting portion of the analysis and includes results on finite-sample concentration of order statistics, which, to the best of our knowledge, are novel. 
We discuss the details of the sample complexity and success probability of this result and argue the correctness of the algorithm in Section~\ref{sec:details_bucketing}, Section~\ref{ch:main_result}, and Section~\ref{sec:concentration_of_os}. To our knowledge, this is the first algorithm with finite sample guarantees to test the shape of the tail weight of a distribution with unbounded support.

Indeed, the novelty of our finite-sample guarantee is underscored by our result showing that without some structural assumptions, no finite sample algorithm can succeed. In particular, in Section~\ref{sec:hardness-main-body}, we develop two classes of distributions, one light-tailed and one that is heavy-tailed over a large but non-contiguous portion of the support, that cannot be distinguished with finitely many samples. The family of heavy-tailed distributions involves embedding hard instances into non-contiguous small parts of a light-tailed distribution, thereby ``fooling'' any finite-sample algorithm into classifying distributions from this class as light-tailed distributions. \par



\paragraph{Summary of our contributions:}
\begin{itemize}[noitemsep, leftmargin=0.15cm]
    \item We give a novel parametrized  definition for heavy-tailed distributions.
    
    
    \item We develop an algorithm for testing heavy-tailed distributions in the property testing setting which uses finitely many samples. We theoretically prove the correctness of our algorithm. 
    
    \item As a byproduct of our result, we develop new results on the concentration of the ordered statistics drawn from an arbitrary distributions while using finitely many samples which may be of independent interest. 
     
    \item We show intractability of this problem without assumptions by presenting two classes of distributions, one light-tailed and one heavy-tailed, that are indistinguishable by any finite-sample algorithm.
    
    \item We provide experimental results which confirm the performance of our algorithm. 
    
\end{itemize}

\paragraph{Related Work} \label{sec:intro_related_work}
Testing the tail weight of a distribution has been studied in the asymptotic theory of statistics by~\cite{bryson_heavy-tailed_1974}, where the weight of the tail is
characterized via a statistic referred to as {\em conditional mean exceedance},  and a proxy for it 
is used to distinguish between the families of Lomax and exponential distributions. 
Several works in the asymptotic theory literature address properties of distributions with monotone 
hazard rate~\cite{barlow_properties_1963} and algorithms that test whether a distribution
has monotone hazard rate (a similar condition to the one used in this work), 
with guaranteed asymptotic convergence~\cite{hall_testing_2005, gijbels__nonparametric_2004}. While these works operate in the asymptotic regime, we address the finite sample setting.

The mostly closely related work to ours comes from the theoretical computer science community. The task of 
distinguishing whether a {\em discrete} distribution has monotone  hazard rate\footnote{\cite{adk_optimal, canonne_testing_2016} use ``monotone hazard rate'' to refer to monotone \emph{increasing} hazard rate.} (MHR), which characterizes ``light-tailed" distributions, or is 
far in $L_1$ distance from such a distribution, has been considered in  
\cite{adk_optimal,canonne_testing_2016}.

In these methods, sample complexity is dependent on the domain size; the method of \cite{adk_optimal} relies on a linear program to learn the weight associated with each discrete element. Thus, these methods do not easily translate to finite sample guarantees over continuous domains. Since the sample complexity of this algorithm scales with domain size, it is not finite in our setting. In contrast, we give a finite sample guarantee, but the result is incomparable due to the assumptions we make. For more details, see Appendix~\ref{appendix:relatedwork}.

\section{Preliminaries} \label{ch:prelim}

\paragraph{Notation}
We say that a distribution has Probability Density Function (PDF) $f(x)$ and 
Cumulative Density Function (CDF) $F(x)$. 
Some families of distributions we reference are:
Lomax $\left(f(x) = \frac{\alpha}{\lambda} \left[ 1 + \frac{x}{\lambda}\right]^{-(\alpha + 1)}\right)$,
exponential $\left(f(x) = \lambda e^{-\lambda x}\right)$, 
and half of a Gaussian $\left(f(x) = 2/\sigma \sqrt{2 \pi}) e^{-\frac{1}{2} \left(\frac{x}{\sigma}\right)^2}\right)$, 
all on $[0, \infty)$. For more details regarding them, see Appendix~\ref{appendix:pdf_cdf_details}. 
A set of $n$ samples drawn from a distribution with PDF $f$ is denoted $x_1, x_2, ..., x_n$. 
The $i^{\text{th}}$ order statistic, the random variable representing the $i^{\text{th}}$ smallest sample of this set of $n$ samples, is denoted $X_{(i)}$. If we use $m$ buckets, then the bucket endpoints in terms of the order statistic are $X_{\left(\frac{n}{m}\right)}, X_{\left(\frac{2n}{m}\right)}...$.

\paragraph{Defining ``Heavy-Tailed''} 
While several definitions of ``heavy-tailed" have been proposed,
we
consider definitions based on the hazard rate of the distribution.
The hazard rate (HR) is defined in~\cite{su_characterizations_2003} as $HR(x) = \frac{f(x)}{1 - F(x)}$. The textbook of~\cite{klugman_loss_2004} defines
a distribution with hazard rate that is a decreasing function of $x$ to be ``heavy-tailed.'' We consider this definition because it succinctly captures what many other definitions address obliquely --  having a tail that decays more slowly than exponential.
Thus, this definition captures the fact that in a heavy-tailed distribution, the probability mass at a given point relative to the tail decreases further into the distribution. 
This leads us to introduce the following parameterization of how heavy-tailed a distribution is.
\begin{definition}\label{def:alpha-Heavy-Tailed}
We say a distribution with PDF $f(x)$ and CDF $F(x)$ is \emph{$(\alpha, \rho)$-Heavy-Tailed} 
with $\alpha > 0, 0 < \rho < 1$ if the derivative of the hazard rate is negative with magnitude
at least $\alpha$ on
some interval $[x_1, x_2]$, where at least $\rho$ of the probability mass lies in that interval. 
That is, if $H_f'(x) < -\alpha ~\forall x \in \left[x_1, x_2\right] ~s.t.~ F(x_2) - F(x_1) \ge \rho$,
we call $f(x)$  $(\alpha, \rho)$-Heavy-Tailed.  
\end{definition}

We use this definition in Claim \ref{claim:detection_of_ht}, showing that our algorithm will detect any $(\alpha, \rho)$-heavy-tailed distribution with a sample complexity dependent on both. The gap between light-tailed distributions and an $(\alpha, \rho)$-heavy-tailed distribution is related directly to $\alpha$, so as $\alpha \rightarrow 0$, the heavy-tailed distributions become harder to distinguish from light-tailed ones. Meanwhile, the parameter $\rho$ considers what percent of the mass lies in a heavy-tailed region. A wide class of distributions of interest are $(\alpha, \rho)$-heavy-tailed for some $\alpha, \rho$.
For example, the general class of distributions with CDF $F(x) = 1 - e^{-\gamma x^m}$, 
where $0 < m < 1$ and PDF $f(x) = \lambda m x^{m-1} e^{-\lambda x^m}$  is $(m\cdot (1-m), 1-e^{-1})$-heavy-tailed. 

Consequent to the definition of heavy-tailed found in~\cite{klugman_loss_2004}, we have the following definition for light-tailed. Exponential and half-Gaussian distributions are light tailed according to this definition.

\begin{definition} \label{defn:light_tailed}
We say a distribution with PDF $f(x)$, CDF $F(x)$ is \emph{light-tailed} if its hazard rate $\frac{f(x)}{1 - F(x)}$ is non-decreasing, or equivalently, if the derivative of the hazard rate is non-negative.
\end{definition}

\paragraph{Defining ``Well-Behaved'' distributions}  
As discussed earlier when describing our setting, we consider continuous distributions on $R_+ = [0, \infty)$ that are monotone decreasing, that is, $f(x_1) \ge f(x_2)$ when $x_1 < x_2$, so that there is a clear notion of the ``tail''\footnote{Note that the tester can also handle a unimodal distribution by determining where the decreasing region is. The increasing region could be handled by reflecting it across the mode.}. Likewise, we require boundedness and smoothness conditions to make the problem tractable but not trivial. Indeed, in Appendix~\ref{sec:lowerbound}, we show a hardness result implying that without any regularity conditions, no finite sample algorithm can succeed at this problem. In the literature, it is common to assume Lipschitz density function. However, bounding the derivative of the density implies that functions that drop too quickly cannot be considered, significantly limiting the kinds of functions we could test. Thus, instead, we assume that the inverse CDF of distributions satisfy continuity and Lipschitz conditions.
 Lipschitzness assumptions are also needed due to
the finite sample regime: with infinitely many samples, we could rely on continuity alone. We formalize and unify these constraints in the following definition, and in this paper we consider distributions satisfying this property. 

\begin{definition} \label{def:well_behaved}
Let $f$ be a distribution on the range $R_+ = [0, \infty)$ with CDF $F$. We say $f$ is \emph{well-behaved} if the following conditions hold: 
\begin{itemize}[noitemsep, leftmargin=0.25cm]
    \item $f$ is a non-increasing continuous function over $R_+$, and $f'(x)$ exists for every $x \in R_+$. 
    \item For all $x$ in the domain, $f(x)$ is bounded by a constant $\beta$. 
    \item For every $x \in R_+$, $F'(x)$ exists, and it is equal to $f(x)$.
    \item The first and second derivatives of $F^{-1}(x)$ are Lipschitz with parameters $B_1, B_2$, respectively, on the domain until the very end (where it cannot be Lipschitz anyway\footnote{For a distribution over an unbounded domain, the CDF asymptotically approaches 1 as the domain value reaches $\infty$. Thus, its inverse must change in an unbounded way as $x \rightarrow 1\,.$}). We will call the domain over which it is Lipschitz $[0, 1-\zeta]$ and explain what $\zeta$ must be in Section~\ref{ch:main_result}. 
\end{itemize}
\end{definition}
The method we present in this paper relies on accurately estimating quantiles, and this is where the  particular smoothness assumptions above are used in our work. It is not clear whether this assumption or reliance on quantiles is necessary in general, and this is interesting to investigate in future work.


\paragraph{Problem Statement} \label{sec:para_problem}
We assume we have have sample access to the distribution: that is, when we query the distribution, we are returned one iid sample from the distribution. 
Consider class $\cal L$, the class of light-tailed distributions and class $\cal H_{\alpha, \rho}$, the class of $(\alpha, \rho)$-heavy-tailed distributions. Our goal is to develop an algorithm to determine whether a distribution comes from $\cal L$ or $\cal H_{\alpha, \rho}$ with probability 0.9 using finitely many samples, with sample complexity dependent on $\epsilon, \alpha, \rho$ and parameters of the distribution.

\section{A Proxy for the Behavior of the Tail} \label{sec:details_bucketing}
In this section, we motivate and define the Equal Weight Bucketing Scheme, upon which we build our test statistic and algorithm. We first determine the endpoints of the buckets and derive their lengths and rates of change of lengths based on the PDF and CDF. Then, we provide a proxy quantity based on this bucketing scheme that is equivalent to testing the hazard rate condition for heavy-tailedness. Finally, we explain how we can calculate this test statistic from samples by describing the bucketing scheme in terms of order statistics.\par 
\subsection{Defining a Bucketing Scheme} \label{sec:proxy_defn_buckets}

Dividing the support of a distribution into buckets is a well-known approach in distribution testing and is used widely in the literature~\cite{birge_risk_1987, batu_sublinear_2004, canonne_sampling_2018, batu_testing_2000}.
In essence, by considering the distribution on these buckets, we often can see trends that are less susceptible to sampling noise. In order to capture the drop rate of the distribution, we propose an equal-weight bucketing scheme. 
\par
We derive the expressions for the continuous notion of the Equal Weight Bucketing Scheme. The ``weight'' refers to probability mass.
The left endpoint of a bucket starting at $y$ with weight $dy$ is given by $\mathbb{I}(y)$ and the length of that bucket (rate of change of endpoints) is given by by $\mathbb{L}(\cdot)$. This continuous notion allows us to derive a proxy quantity that reflects the hazard rate condition.
The endpoints of the buckets are determined by the inverse of the CDF, giving us that in general, $
\mathbb{I}(y) \coloneqq F^{-1}\left(y\right) $
The derivative $d\mathbb{I}/dy$ gives us the length of the intervals. 
By the chain rule, the lengths and derivatives of lengths of the intervals are:
\begin{align} \label{eqn:L_of_bucket}
\mathbb{L}(y) \coloneqq \frac{d}{dy}F^{-1}(y) = \frac{1}{F'(F^{-1}(y))} = \frac{1}{f(F^{-1}(y))} \quad \text{and} \quad
    \frac{d}{dy}\mathbb{L}(y) \coloneqq \frac{-f'(F^{-1}(y))}{f(F^{-1}(y))^3}
\end{align}

\subsection{Derivation of Proxy Quantity} \label{sec:proxy_exact_test}
Next, we present the proxy quantity that verifies the hazard rate condition for light-tailedness and $(\alpha, \rho)$-heavy-tailed. This quantity is based on partitioning the support into intervals on which the density incurs  equal weight. The full proof can be found in Appendix~\ref{appendix:analysis-stat-exact}.

\begin{restatable}{theorem}{hrstatthm}
 \label{thm:hr-test-statistic}
For $\mathbb{L}(x)$ (defined in Equation~\ref{eqn:L_of_bucket}) for a well-behaved function $f$, define $S(z) \coloneqq \frac{\mathbb{L}(z)}{\frac{d}{dz}\mathbb{L}(z)}$:
\begin{itemize}[noitemsep]
    \item If for all $z \in [0, 1]$:
$ S(z) > 1 - z $, then the underlying distribution is light-tailed by Definition~\ref{defn:light_tailed}.
    \item Otherwise, if $S(z) < 1- z - \frac{\alpha (1-z)^2}{\beta^3 B_1}$ 
    for $z \in [z_0, z_0 + \rho]$, for any $z_0\,,$ then it is $(\alpha, \rho)$-heavy-tailed.
\end{itemize}
\end{restatable} 

\begin{definition} \label{defn:gap}
We refer to the distance in the proxy quantity between the lightest $(\alpha, \rho)$-heavy-tailed distribution and the heaviest light-tailed distribution (exponential), a lower bound on which is $\frac{\alpha (1-z)^2}{\beta^3 B_1}$, as the {\em gap}.
\end{definition}
This is the gap in the proxy quantity metric between the two classes of distributions we hope to distinguish between. Since our eventual algorithm will rely on samples, the gap gives us some slack with which to handle error.

\subsection{Test Statistic in Terms of Buckets} 
\label{sec:proxy_buckets_discrete}
Here, we explain how we convert the proxy quantity into something we can calculate from knowing bucket endpoints. We approximate the derivative by the difference quotient. For well-behaved distributions,  this approximation does not incur too much error due to  Lipschitz-ness on the domain of interest. Detailed proofs can be found in Appendix~\ref{appendix:analysis_stat}. 

\begin{restatable}{fact}{derivapproxlemma}
\label{lem:derivative_approximation}
When the derivative of a function $g$ is $B-Lipschitz$,
and the derivative $g'(y)$ is monotone,
then approximating $g'(y)$ by the difference quotient $\frac{g(y+\Delta y) - g(y)}{\Delta y}$ incurs no more than $B\Delta y$ additive error.
\end{restatable}
This yields the following lemma, used to relax derivatives in $S$ to discrete derivatives.\par

\begin{restatable}{lemmma}{corsecondderiv}
\label{cor:derivative-noisy-approximation}
When the derivative of a function $g$ is $B_1$-Lipschitz, the second derivative is $B_2$-Lipschitz,
the derivative $g'(y)$ and the second derivative $g''(y)$ are both monotone,
then approximating $g''(y)$ by:
\begin{enumerate}[noitemsep]
    \item estimating $\tilde g'(y) = \frac{g(y + \Delta y_1) - g(y)}{\Delta y_1}$,
    and $\tilde g'(y+\Delta y') = \frac{g(y+\Delta y_2 + \Delta y_1) - g(y+\Delta y_2)}{\Delta y_1}$,
    \item and estimating $\tilde g''(y) = \frac{\tilde g'(y + \Delta y_2) - \tilde g'(y)}{\Delta y_2}$.
\end{enumerate}
incurs no more than $2B_1\frac{\Delta y_1}{\Delta y_2} + B_2\Delta y_2$ additive error.
\end{restatable}

According to Fact \ref{lem:derivative_approximation}, we can get bounded approximation error while approximating the derivative with steps of size $1/k$, which we can make small by setting $k$ appropriately. However, in order for the second derivative approximation discussed in Lemma~\ref{cor:derivative-noisy-approximation} to also be small, we need to consider buckets at two different levels of granularity, so we set $\Delta y_1 = \frac{1}{k^2}$ and $\Delta y_2 = \frac{1}{k}$. This gives us additive error of $\frac{B_1}{k}$ in the numerator and additive error of $\frac{2B_1 + B_2}{k}$ in the denominator, which we can make small by setting $k$ appropriately (see Corollary~\ref{cor:bucket_bound_approximation}). Thus, we can approximate the proxy quantity by a discrete equivalent without incurring too much error (quantified in Lemma~\ref{lem:s_tilde_error}). Accordingly, we define:
\begin{align}
    \tilde{S} &\coloneqq \frac{\tilde L_1}{(\tilde L_2 - \tilde L_1)/\Delta y_2} \,,
    \text{ where } \, 
    \tilde L_1 \coloneqq \frac{\mathbb{I}(y + \Delta y_1) - \mathbb{I}(y)}{\Delta y_1} \, \text{and }\tilde L_2 \coloneqq
    \frac{\mathbb{I}(y + \Delta y_1 + \Delta y_2) - \mathbb{I}(y + \Delta y_2)}{\Delta y_1}\,.
	\label{eqn:delta2set} 
\end{align}

\subsection{Test Statistic in Terms of Order Statistic} \label{sec:tester_os}
In this section, we extend the formulation of the statistic in Equation~\ref{eqn:delta2set} to show how we calculate it from samples. We set $y = i/k$. Approximating the derivative as the difference divided by the length, we get that the aforementioned tester can be written as follows in terms of the order statistic:
\begin{align*}
\label{eq:def_hat_S} 
 \hat{S}[i] &= \frac{\left(X_{\left(\frac{ik + 1}{k^2} \cdot (n+1)\right)} - X_{\left(\frac{i}{k} \cdot (n+1)\right)}\right)}{ \Bigg(k\bigg(X_{\left(\frac{(i + 1)k + 1}{k^2} \cdot (n+1)\right)} - X_{\left(\frac{i + 1}{k} \cdot (n+1)\right)}  - X_{\left(\frac{ik + 1}{k^2} \cdot (n+1)\right)} + X_{\left(\frac{i}{k} \cdot (n+1)\right)}\bigg)\Bigg)}\,.
\end{align*}

In order to calculate the endpoints of the equal weight buckets, we draw four sets of samples, sort them, and then determine the bucket endpoints by considering the samples at indices $\frac{i}{k^2} \cdot n; i \in \{0, 1, ... k^2\}$. In any given calculation of the statistic, we need four of these order statistics;  using different splits for each results in independence. \par

The test statistic is sensitive to the endpoints due to reliance on the length of buckets; since a small amount of mass could lie in an interval with very long length, this concentration of the endpoints is challenging to show (addressed in Appendix~\ref{sec:concentration_of_os}). \par

\section{Main Result} \label{ch:main_result}
In this section, we present our main result (Theorem~\ref{thm:main}), the algorithm that gives us that upper bound, and discuss an overview of the proof.

\begin{restatable}{theorem}{mainthm} \label{thm:main}
There exists an algorithm that distinguishes between $(\alpha, \rho)$-heavy-tailed distributions and light-tailed distributions requiring $\Theta\left(\max \left\{ \frac{\beta^3 B_1}{\alpha \rho^2}  , k \right\} \cdot k^2 \log k \sqrt{\sqrt{B_1} + 1} \right)$ samples with success probability 9/10, where $k = \max\left\{\Theta\left(\frac{\beta^4 B_1 (2 B_1 + B_2)}{\alpha \, \rho^2}\right), \frac{4}{\rho} \right\}$.\footnote{This can be increased to probability $1-\delta$ by repeating the algorithm $\log 1/\delta$ times using the standard amplification technique.}
Such an algorithm is given in Algorithm~\ref{alg:tester}. 
\end{restatable}

 \begin{wrapfigure}{l}{0.5\textwidth}
 \vspace{-9 mm}
 \begin{minipage}[t]{.5\textwidth}
 \begin{algorithm}[H]
 \footnotesize
 \caption{$(\alpha, \rho)$-Heavy-Tailed Test}  \label{alg:tester}
 \begin{algorithmic}[1]
 \State Draw four sets of $n$ samples $\mathcal{R}^{(1)}, \mathcal{R}^{(2)}, \mathcal{R}^{(3)}, \mathcal{R}^{(4)} \leftarrow$ from the distribution 
 \State Sort the samples in $\mathcal{R}^{(1)}, \mathcal{R}^{(2)}, \mathcal{R}^{(3)}, \mathcal{R}^{(4)}$
 \State Split each $R^{(l)}$ into $k^2$ equal weight buckets.
 \State $\forall\ \ i,j < k$, determine the interval endpoint $I^{(l)}[i,j]$ corresponding to the $(i\cdot k + j)^{th}$ bucket in $R^{(l)}$ (which is order statistic $X_{(i \cdot k + j)}$).
 \State Calculate $L_1[i] = I^{(1)}[i, 1] - I^{(2)}[i, 0]$ and $L_2[i] = I^{(3)}[i, 1] - I^{(4)}[i, 0]$.
 \State Calculate $L'[i] = \frac{L_1[i+1] - L_2[i]}{(1/k)}$.
 \State Calculate the statistic $\hat{S}[i] = \frac{L_1[i]}{L'[i]}$ for $1 < i < k$.
 \If{$\hat{S}[i] < 1 - \frac{i}{k} - \frac{1}{2}\text{gap}(\alpha)$ for any $i \in \{2, 3, \ldots, k-1\} $ }
  \State PASS.
 \Else
  \State FAIL.
 \EndIf
 \end{algorithmic}
 \end{algorithm}
 \end{minipage}
 \vspace{-5mm}
 \end{wrapfigure}

\paragraph{Algorithm} At a high level, the algorithm draws samples, uses them to estimate the statistic, $\hat{S}$, for each bucket, and compares the statistics with the threshold defined in Theorem~\ref{thm:hr-test-statistic} to determine tail weight. More specifically, the algorithm first computes the order statistics as described in Section~\ref{sec:tester_os}.
Subsequently, the algorithm calculates the lengths $L$ and the change in lengths $L'$ of the buckets,
and computes a test statistic $S[i]$ for $i \in \{2, \cdots, k-1 \}$. The statistic is calculated for every $k^{th}$ bucket within the $k^2$ buckets.
We do not consider the first and last buckets, since the function is not Lipschitz there and so the approximations to the derivative will not be close to the true derivative. The parameter $\zeta$ as described in Definition~\ref{def:well_behaved}, thus, must be $\le 1/2k$, which allows us to safely use all but the last $k-1$ buckets of length $k^2$.
If each test statistic lies above the threshold, the underlying distribution is declared light-tailed.
On the other hand, if any of them lies below the threshold, the distribution is declared heavy-tailed.

%
%
%
%
%
%
%
%

\begin{claim}  \label{claim:detection_of_ht}
If $f(x)$ is $(\alpha, \rho)$-Heavy-Tailed, then Alg.~\ref{alg:tester} will pass it with high probability. Moreover, if the underlying distribution is light-tailed then Alg.~\ref{alg:tester} will fail it with high probability.
\end{claim}

\vspace{-4mm}

\paragraph{Proof Overview}

The proof breaks down into stages that
correspond to the different approximations we make
to get from the proxy quantity to the final test statistic.
We will use $S$ to denote the proxy quantity, $\tilde{S}$ to denote the statistic approximated by discrete derivatives,
and $\hat{S}$ is used to denote the empirical statistic that calculated using order statistics of the samples (Figure~\ref{fig:s_tildes_hats_errors_2}). 

Our analysis proceeds through four stages: (1) correctness of proxy $S$; (2) $\tilde S$ close to $S$; (3) $\hat S$ close to $\tilde S$; (4) setting $k, n$ to satisfy theorem.

Here, we provide the lemmas that quantify the error incurred from $S$ to $\tilde S$ and $\tilde S$ to $\hat S$. Further details and analysis of how to satisfy the conditions of these lemmas, are discussed in Section~\ref{sec:concentration_of_os}. 

\vspace{-3mm}
\subsubsection*{\textsc{Part 1/4: $\pmb{S}$ is an accurate proxy.}}
The proxy quantity derived in Theorem~\ref{thm:hr-test-statistic} considered with respect to the threshold (halfway across the gap) gives a test which accurately determines whether a set of samples came from a light-tailed distribution or a distribution that is $(\alpha, \rho)$-heavy-tailed. An $(\alpha, \rho)$-heavy-tailed distribution has hazard rate decreasing at least by $\alpha$ over a region of the PDF with probability mass $\rho$, which gives us an expression for how far the statistic must be from the original threshold $1 - i/k$ in order for the hazard rate to be decreasing by at least $\alpha$. Further, if the region of mass $\rho$ lies in at least two buckets, then the proxy quantity will detect it. Thus, if a distribution is light-tailed, then \textit{all} $k-3$ of the calculated proxy quantities calculated will lie above $1-i/k$; if even one of the proxies lies below $1 - i/k - gap$, then the distribution is $(\alpha, \rho)$-heavy-tailed. See Section~\ref{sec:details_bucketing} and Appendix~\ref{appendix:analysis-stat-exact} for detailed discussion. \par
\vspace{-3mm}

\subsubsection*{\textsc{Part 2/4: $\pmb{\tilde S}$ is close to $\pmb{S}$.}}

In the next step, we show that approximating the derivatives in the proxy quantity by the respective difference quotients causes the test statistic to incur bounded error. Recall that the proxy we are using is
$
S = N/D = \left( \frac{d}{dy} F^{-1}(y)\right) / \left( \frac{d^2}{dy^2} F^{-1}(y) \right), 
$
which is approximated by $\tilde S =  \tilde N / \tilde D = \tilde L_1/((\tilde L_2 - \tilde L_1)/\Delta y_2)$ as in Eq.~\ref{eqn:delta2set}.

If the error incurred in $\tilde N, \tilde D$ by approximating the derivatives as above has value $\epsilon'$ (we show this condition is met in Section~\ref{sec:concentration_of_os} due to the Lipschitzness conditions), then either we can estimate the value of the proxy quantity within 
a bounded additive error (and the proxy quantity is small) or the value of the proxy quantity is greater than 1 (and we know we are in the light-tailed case).

\begin{restatable}[Additive Bound for $\tilde S$]{lemmma}{lemmaadditivestilde}
\label{lem:s_tilde_error}
Given a parameter $\epsilon < 1$, if $|\tilde N-N|$ and $|\tilde D-D|$ are at most $\epsilon' \coloneqq \epsilon/(6\beta)$, then either $|\tilde S-S|<\epsilon$ or both $S, \tilde S$ are at least one.
\end{restatable}

\vspace{-2mm}

\subsubsection*{\textsc{PART 3/4: $\pmb{\hat S}$ (calculated from order statistics) is close to $\pmb{\tilde S}$.}}

We must next show that we can approximate $\tilde S$ as defined in the previous section by $\hat S\,,$ computed from the order statistics of a set of samples.  First, if we estimate $\tilde L_1$ and $\tilde L_2$ accurately up to a multiplicative $(1 \pm \epsilon')$ factor, and obtain $\hat{L}_1$ and $\hat{L}_2$, then $\tilde{S}$ can be approximated by
$\hat{S} \coloneqq \frac{\hat{L}_1}{k \cdot \left(\hat{L}_2 - \hat{L}_1\right)}\,.$

\begin{restatable}[Multiplicative Bound for $\hat S$]{lemmma}{lemmabuckettostaterror} \label{lem:bucket_to_stat_error}
Suppose we have $\hat{L}_1$ and $\hat{L}_2$, the estimates of $\tilde L_1$ and $\tilde L_2$ with a multiplicative factor of $\epsilon' = \min \left( \Theta\left(\epsilon/((1+\epsilon)\cdot k)\right), \Theta\left(1/k^2\right) \right)$.
Then, one of the following cases holds:
\vspace{-3mm}
\begin{enumerate}[noitemsep]
    \item $\hat{S} > 1 - 2/k$ and $\Tilde{S} > 1$, implying they come from a light-tailed region of the distribution.
    \item $(1-\epsilon) \cdot \Tilde{S} \leq \hat{S} \leq (1+\epsilon) \cdot \Tilde{S}$.
\end{enumerate}
\end{restatable}

To show that we can get the multiplicative estimates $\hat L_1, \hat L_2$, we need to show that the order statistics concentrate well so that the estimates for the end points are not too far off from their theoretical values. For this, we construct a map between samples from the uniform distribution on $[0, 1]$ and an arbitrary distribution with CDF $F$, showing that concentration of the order statistics of a set of samples from the former implies concentration of order statistics of a set of samples from the latter. This concentration can be expressed in terms of additive error (Lemma~\ref{lem:additive_error_x}), and this can be translated to multiplicative error (Lemma~\ref{lemma:samples}). We discuss this in more detail in Section~\ref{sec:concentration_of_os}.

\subsubsection*{\textsc{Part 4/4: Errors can be set to satisfy theorem.}}
Finally, we must limit the errors we incur in Parts 2 and 3 to determine how many buckets and samples we require. For this, we note that after incurring both the derivative approximation error and the sampling error, the test statistic $\hat S$ must still be on the same side of the threshold (partway across the gap), as $S$. 
    We ensure that we split the distribution into sufficiently many buckets that no more than a quarter of the gap is crossed due to derivative approximation error. Further, we draw enough samples that no more a tenth of the gap is crossed due to sampling error. Thus, even when both errors are incurred, the test statistic remains on the correct side of the threshold.

\section{Hardness Result}
\label{sec:hardness-main-body}

In this section, we show that in the absence of assumptions, this problem cannot be solved with finitely many samples. In particular, for any number of samples $m$, we present two classes of distributions, one light-tailed and the other heavy-tailed, that are indistinguishable using $m$ samples. To start, we consider a slightly different definition for heavy-tailed-ness. We then show that it is hard to distinguish these distributions from light-tailed ones.

\begin{definition}
We say a distribution $p$ is  $(\alpha, \rho)$-scattered-heavy-tailed if the hazard rate is decreasing by rate $\ge \alpha$ over measurable intervals of the domain with probability mass $\ge \rho$. 
\end{definition}

\paragraph{High level idea:} We construct two classes of distributions that are hard to distinguish with few samples:  $\CC_L$ (light-tailed), and  $\CC_H$ (heavy-tailed).  The class of light-tailed distributions contains only one member that is an exponential distribution with $\fE(x) = e^{-x}$. We construct the class of heavy-tailed distributions via a randomized process as follows: We start off by the same distribution $\fE(x) = e^{-x}$. We split the domain of $\fE$ into $s$ {\em chunks} such that the probability mass in every chunk is equal to $1/s$. Then, we select roughly $\rho' = \Theta(\rho)$ fraction of these chunks randomly and embed a heavy-tailed distribution, namely $\fH$, in (some of) those selected chunks. The construction has two key properties: First, the probability mass of a chunk remains the same even when the alteration happens. Second, if we draw one sample from a chunk, we cannot tell whether it is altered or not. 

When we alter a chunk, we randomly replace $\fE$ by a heavy-tailed piece $\fH$ or another partial PDF $\fbar$. We simply define $\fbar$ such that the mixture of $\fH$ and $\fbar$ each with probability a half gives us exactly $\fE$. Thus, if we receive one sample from a chunk that comes from a random $\CC_H$, it is impossible to tell whether the chunk is altered or not. It is worth noting that this process generates a class of distributions, $\CC_H$, that depends on a parameter $s$. We may also $\CC_{H}(s)$ to denote it. To complete our proof, we show that for any algorithm that uses $m$ samples, there is a sufficiently large $s$ such that it is very unlikely to have more than one sample per chunk. Thus, $\CC_L$ and $\CC_{H}(s)$ are indistinguishable when we use $m$ samples. This fact implies that no algorithm that uses finitely many samples can distinguish a light tailed distribution from a distribution that is heavy on measurable subset of the domain with mass $\rho$ unless we make further assumptions including that the heavy-tailed part might need to be contiguous. We state the result formally in the theorem below, and the proof is in Appendix~\ref{sec:lowerbound}.

\begin{theorem} 
For any integer $m$, there is no algorithm that receives $m$ samples from $p\,,$ a monotone and continuous distribution, and can distinguish whether $p$ is light-tailed or $(\alpha, \rho)-$scattered-heavy-tailed for $\alpha < 0.0043$ and $\rho < 0.5$ with probability more than $0.5 + o(1)$.
\end{theorem}

\section{Experiments}
\label{sec:experiments}

We present experiments that validate our theoretical results. Through experiments on synthetic data, we demonstrate that our statistic distinguishes between Gaussian (light-tailed) and Lomax (heavy-tailed) distributions. The algorithm we run has weaker theoretical guarantees, as we do not consider both granularities of bucketing during derivative approximation.

\paragraph{Data and Methods}
Points are sampled from (half) Gaussian  and Lomax  distributions using built-in functions in \texttt{numpy.random}. We sample $n = \Theta(k^4) \approx 51$ million and $n = \Theta(k^5) \approx 300$ million points for $k = 32$. 

\paragraph{Results} We find that the test statistic distinguishes between Gaussian and Lomax over the required range of bucket indices. In Figure~\ref{fig:expt_results}, the dashed red lines refer to the calculated value of the test statistic for a Gaussian distribution and blue lines for the Lomax. One standard deviation away from the mean is shaded in the appropriate color. From Figure~\ref{fig:expt_results}, we note: (1) the value of the test statistic calculated from samples is very close to the proxy quantity ($S$); (2) in part (a) (approx. 51 million samples), we can distinguish Gaussian (red) and Lomax (blue) well over a large range of buckets, but we see greater variability between runs (spread of the dashed lines) than in part (b) (300 million samples, same number of buckets), where 
the calculated statistic concentrates better. In Appendix~\ref{appendix:sec-low-sample}, we explore what happens when we use substantially fewer samples, observing that we require some minimum number of samples per bucket but start to see separation between the same two distributions with a relatively small number of buckets. Thus, in situations where it is not essential that we get provable guarantees, the same algorithm could work with fewer samples.

\begin{figure*}[h!]
\centering
    \centering
    \begin{tabular}{{p{0.45\textwidth}p{0.45\textwidth}}}
     \includegraphics[width=0.4\textwidth]{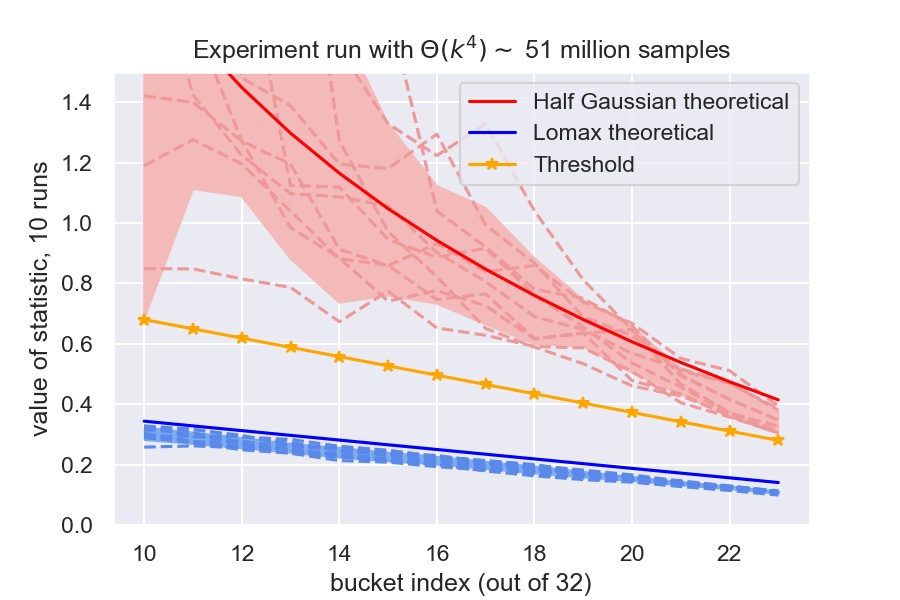} & \includegraphics[width=0.4\textwidth]{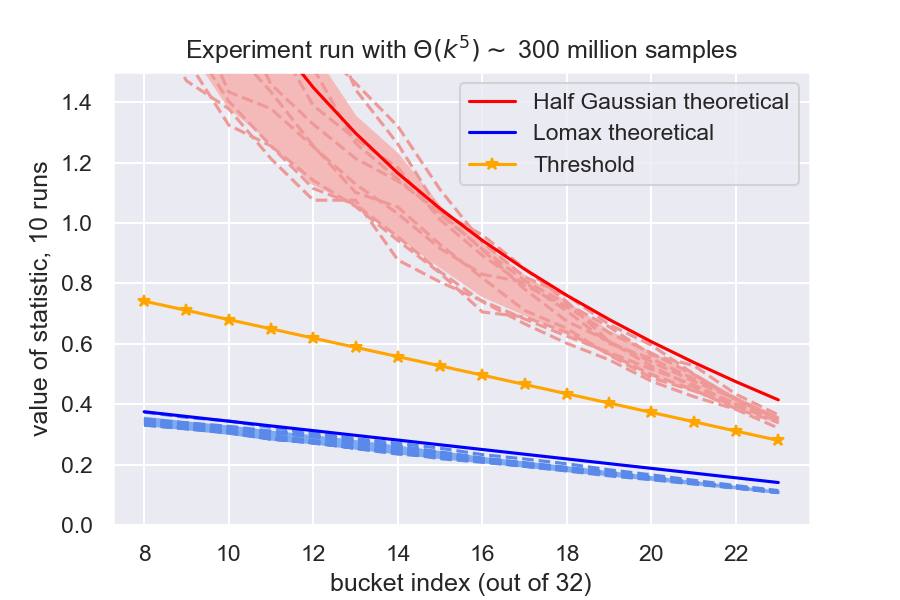} \\
         (a) & (b) \\
    \end{tabular}
    
        \caption{\scriptsize These plots depict the results of 10 runs of Algorithm~\ref{alg:weak_tester} on samples drawn from a half Gaussian distribution and 10 runs of it on samples drawn from a Lomax. Experiments plotted in (a) used $n = \Theta(k^4)$ samples, and those in (b) used $n = \Theta(k^5)$ samples. With more samples, we are able to distinguish the two distributions over a broader range of buckets. Depicted are the proxy quantities (solid), lines representing the 10 runs for each distribution (dashed), and shading one standard deviation above and below the mean. The orange line represents the threshold, calculated based on $\alpha=1/4$ for Lomax and appropriate settings for $\beta, B_1$ based on the distributions we considered.}
    \label{fig:expt_results}

\end{figure*}

\medskip

\small

\bibliographystyle{alpha}
\bibliography{bibliography}

\onecolumn
\appendix

\section{Alternate Definitions of Heavy-Tailedness} \label{appendix:defns_ht}
It is non-trivial to give a single unifying definition of heavy-tailed distributions, 
since the definition needs to accommodate a wide range of behaviors, including distributions whose tails fall off at irregular rates. Further, no notion of having a ``heavy'' tail is absolute -- there must be a point of reference.
Throughout the literature, a plethora of non-equivalent conditions 
on the tail of a distribution are used to define what it means for the distribution to be considered
``heavy-tailed.'' 
Some examples include increasing conditional mean
exceedance~\cite{bryson_heavy-tailed_1974}, decreasing hazard rate~\cite{klugman_loss_2004}, 
 the lack of finite exponential moments, having tails that decay more slowly than exponential~\cite{mikosch_regular_1999, chistyakov_subexp_1964}, infinite variance~\cite{harchol-balter_e_nodate}, regularly varying tails~\cite{mikosch_regular_1999}, 
and definitions that build on these conditions~\cite{su_characterizations_2003}. One uniting factor of these definitions is that the point of reference is the exponential distribution, meaning that a distribution whose tail decays more slowly than the exponential is considered ``heavy-tailed.'' \footnote{Distributions with regularly varying tails represent a subset of distributions whose tails decay more slowly than exponential.}
In this work, we focus on Klugman et al.'s definition of heavy-tailed that is based on the property that the 
hazard rate of a heavy-tailed distribution is 
decreasing~\cite{klugman_loss_2004}; the characterization is of similar structure to the definitions that~\cite{bryson_heavy-tailed_1974, su_characterizations_2003} use,
admits a clean description, and reflects the idea that heavy tails decay more slowly than exponential
tails, thereby adequately classifying known families of distributions
into heavy-tailed versus light-tailed -- e.g., Lomax (heavy), exponential (light), and Gaussian (light). \par
In this section, we provide alternate definitions for the notion of ``heavy-tailed.'' We use the usual notation that a distribution has Probability Density Function (PDF) $f(x)$ and Cumulative Density Function (CDF) $F(x)$.
\subsection{Increasing Conditional Mean Exceedance} \label{sec:def_cme}
The conditional mean exceedance (CME) is defined in~\cite{bryson_heavy-tailed_1974} to be:
\begin{equation}
\label{defn:cme}
    CME(x) = \mathbb{E}\left[ (X - x) | X \ge x\right] = \frac{\int_x^{\infty} (s - x) f(s) ds}{1 - F(x)}
\end{equation}
According to~\cite{bryson_heavy-tailed_1974}, a distribution is considered heavy-tailed if the CME is increasing for all $x$. The CME essentially captures the shape of tail by considering how the expectation of the tail of the distribution evolves throughout the tail. Thus, if it is increasing, the further into the distribution we go, the more that remains. \par
This definition considers an exponential distribution the canonical ``medium-tailed'' distribution and the Lomax distribution the canonical ``heavy-tailed'' distribution. In practical settings, this metric is valuable, as it represents ``decreasing failure rate''~\cite{harchol-balter_e_nodate}: the longer a job has been alive (as $x$ increases), the longer it is expected to stay alive ($\mathbb{E}(X-x | X \ge x)$ increasing). This definition can also be framed as the reciprocal of the hazard rate of the equilibrium distribution function of a distribution~\cite{su_characterizations_2003}.

\subsection{Infinite Variance}
Harchol-Balter says that heavy-tailed distributions are ones that have infinite variance~\cite{harchol-balter_e_nodate}. From samples, the variance can never really be infinite; in those cases, very large variance signify heavy-tailedness. This definition is difficult to quantify. \par
This definition considers a subset of Lomax and Pareto distributions ``heavy-tailed.''

\subsection{Regular Variation}
\label{sec:def_regvar}
The tail of a distribution captures how the PDF behaves as $x \rightarrow \infty$. We introduce~\cite{mikosch_regular_1999}'s notion of regular variation as a way to capture this behavior.
\begin{definition}
A positive measurable function $f$ is called {\em regularly varying (at infinity) with index $\alpha \in \mathbb{R}$} if
\begin{itemize}
    \item It is defined on some neighborhood $[x_0, \infty)$ of infinity.
    \item $
    \lim \limits_{x \rightarrow \infty} \frac{f(tx)}{f(x)} = t^\alpha \text{ for all } t > 0 \,.
    $
\end{itemize}
If $\alpha = 0$, then $f$ is said to be {\em slowly varying (at infinity) }
\end{definition}
One class of heavy-tailed distributions is that which have regularly varying tails. This includes distributions of the type $x^\alpha$, etc. Note that these are related to generalized Pareto distributions. Thus, this class of distributions is narrower in scope than the set of heavy-tailed distributions considered by other definitions.

\subsection{Moment Generating Function Infinite}
According to~\cite{su_characterizations_2003}, a distribution is considered heavy-tailed if it has the moment generating function (MGF) is not bounded, i.e., the distribution has no finite exponential moments.. That is, if
\begin{align}
    \int \limits_{0}^{\infty} e^{tx} f(x) dx = \infty \quad\quad \forall t > 0
    \,,
\end{align}
then we consider a distribution heavy-tailed. This definition captures the idea that the tail of a distribution must be heavier than exponential in order for it to be considered ``heavy-tailed.''

\subsection{Relating Definitions}

The characterization using the moment generating function
in the previous section considers any distribution of the form $e^{-g(x)}$ where $g(x)$ is asymptotically smaller than linear ``heavy-tailed.'' Indeed, this is true of the CME and HR characterizations, as well, and~\cite{mikosch_regular_1999} notes that distributions that decay more slowly than the exponential are heavy-tailed. For the general class of distributions with CDF $F(x) = 1 - e^{-g(x)}$ and PDF $f(x) = g'(x) e^{-g(x)} = e^{-(g(x) - \ln(g'(x)))}$, the CME characterization, the HR characterization, and the moment-generating function characterization all consider the same families of distributions heavy-tailed and light-tailed. In light of this, we choose to consider the hazard rate definition, as it is the simplest expression of this notion.

\section{Related Work Discussion} \label{appendix:relatedwork}
 In the works by \cite{adk_optimal} and \cite{canonne_sampling_2018},
algorithms are given to perform the task of testing monotone hazard rate with sample complexity that has dependence
on the domain size that is nearly square root.

In them, the domain size is finite and no assumptions on the monotonicity, 
continuity, or Lipschitzness of the distributions are made.   
The paradigm for testing MHR in~\cite{adk_optimal, canonne_testing_2016} 
can be broken into the three steps: first, approximately 
learning the underlying distribution from samples (assuming that it is MHR), 
second, testing whether the learned distribution 
is in fact close to the original distribution, and third, testing whether the 
learned distribution is in the class or far from the class.
The sample complexity ($O(n / \epsilon^2 + \log(n/\epsilon)/\epsilon^4)$) is dominated by the cost of the second step, while the third step
requires no new samples and is solely computational, via a solution
to a linear programming problem. As the variables for the linear program represent the probability mass function, this step cannot easily translate to distributions over continuous domains. Since the sample complexity of this algorithm scales with domain size, it is not finite in our setting. In contrast, we give a finite sample guarantee, but the result is incomparable due to the assumptions we make.

\section{PDF, CDF of Some Common Continuous Distributions}
\label{appendix:pdf_cdf_details}

In this section, we define the various distributions that are referenced throughout the paper as are relevant in our setting.

\paragraph{Exponential} \label{cont_exp_pdf} The cutoff between heavy-tailed and light-tailed distributions according to both the CME and HR definitions is the exponential distribution. The continuous exponential distribution requires parameter $\lambda > 0$ has probability density function (PDF) $f(x) = \lambda e^{-\lambda x}; x \in [0, \infty)$, cumulative density function (CDF) $F(x) = 1 - e^{-\lambda x}$, and quantile function $F^{-1}(x) = \frac{-1}{\lambda} \ln{1-x}$.

\paragraph{Lomax} \label{cont_lom_pdf} 
The Lomax distribution requires parameters $\alpha > 0, \lambda > 0$ and has PDF$f(x) = \frac{\alpha}{\lambda} \left[ 1 + \frac{x}{\lambda}\right]^{-(\alpha + 1)}$, CDF $F(x) = 1 - \left[ 1 + \frac{x}{\lambda}\right]^{-\alpha}$, and quantile function $F^{-1}_{X,L}(x) = \lambda \left( \frac{1}{(1-x)^{1/\alpha}} -1 \right)$. We consider this the canonical ``heavy-tailed'' distribution. 

\paragraph{Half-Gaussian}
The half-Gaussian distribution requires parameter $\sigma$ and has PDF:
\begin{equation}
    f(x) = \frac{2}{\sigma \sqrt{2 \pi}} e^{-\frac{1}{2} \left(\frac{x}{\sigma}\right)^2}; x \in [0, \infty)
\end{equation}
and CDF:
\begin{equation}
    F(x) = \erf{\left(\frac{x}{\sigma \sqrt{2}}\right)}
\end{equation}
The quantile function is:
\begin{equation}
    F^{-1}(x) = \sigma \sqrt{2} \erf^{-1}(x)
\end{equation}
We consider this the canonical ``light-tailed'' distribution.

\section{Test Statistic Analysis} \label{appendix:analysis_stat}
\subsection{Proof of \texorpdfstring{Theorem~\ref{thm:hr-test-statistic}}{Section \ref{thm:hr-test-statistic}}} 
\label{appendix:analysis-stat-exact}
\hrstatthm*
\begin{proof}
Let $f(y)$ be the PDF of a distribution and $F(y)$ the CDF. From Equation~\ref{eqn:L_of_bucket}, 
when considering the ratio of the length of the buckets to the rate of change of the derivatives, we have that:
\begin{align}
    \frac{L(y)}{\frac{d}{dy} L(y)} = \frac{L(y)}{-f'(F^{-1}(y)) f(F^{-1}(y))^{-2} L(y)} = \frac{-f(F^{-1}(y))^{2}}{f'(F^{-1}(y))}\bigg\rvert_{y = z} = \frac{-f(F^{-1}(z))^{2}}{f'(F^{-1}(z))}
\end{align}

First, in order for the distribution to be light-tailed, $\frac{d}{dy}HR(f(y)) > 0$, which implies:
\begin{align}
    \frac{d}{dy}\frac{f(y)}{1-F(y)} &= \frac{(1-F(y)) f'(y) - f(y)(-f(y))}{(1-F(y))^2} > 0 \\
    &\Rightarrow (1-F(y)) f'(y) - f(y)(-f(y)) > 0 \\
    &\Rightarrow  \frac{-f(y)^2}{f'(y)} > 1 - F(y) \bigg \rvert_{y = F^{-1}(z)} \label{eqn:tail_condition}\\
    &\Rightarrow \frac{-f(F^{-1}(z))^2}{f'(F^{-1}(z))} = S(z) > 1 - F\left(F^{-1}\left(z\right)\right) = 1 - z
\end{align}
Equation~\ref{eqn:tail_condition} is a result of our assumption that $f(y)$ is monotone decreasing, and so its derivative must be negative. This derivation gives us a condition on the tail of the distribution which, if satisfied, implies that the distribution is light-tailed. 

Now, we consider the case where the hazard rate decreases by at least $\alpha$. In this case:
\begin{align}
    \frac{d}{dy}\frac{f(y)}{1-F(y)} &= \frac{(1-F(y)) f'(y) - f(y)(-f(y))}{(1-F(y))^2} < - \alpha \\
    &\Rightarrow (1-F(y)) f'(y) - f(y)(-f(y)) < -\alpha (1- F(y))^2 \\
    &\Rightarrow  \frac{-f(y)^2}{f'(y)} < 1 - F(y) + \frac{\alpha (1-F(y))^2}{f'(y)} \bigg \rvert_{y = F^{-1}(z)} \\
    &\Rightarrow \frac{-f(F^{-1}(z))^2}{f'(F^{-1}(z))} = S(z) < 1 - z + \frac{\alpha (1-z))^2}{f'(F^{-1}(z))} = 1 - z + \frac{\alpha (1 - z)^2}{f'(F^{-1}(z))} < 1 - z - \frac{\alpha (1-z)^2}{\beta^3 B_1} 
\end{align}
Thus, if the proxy quantity is greater than $1-z$, then the distribution is light-tailed; otherwise, if the proxy quantity lies below $1-z - \frac{\alpha (1-z)^2}{\beta^3 B_1}$ over a domain of size $\rho$, then the distribution is $(\alpha, \rho)$-heavy-tailed. 
\par
By imposing the condition from Definition~\ref{def:alpha-Heavy-Tailed} on the hazard rate, we were able to recover an expression where one term denotes the usual threshold and there is an additive term, dependent on $\alpha$, that specifies what the \textbf{gap} is between the threshold and the lightest heavy-tailed distribution for which we can test. In particular:
\begin{align*} \label{eqn:alpha_gap}
    S(z) \le \frac{-\alpha (1-z)^2}{-f'(F^{-1}(z))} + (1-z) \Rightarrow \text{gap}(\alpha, z) = \frac{-\alpha (1-z)^2}{-f'(F^{-1}(z))} \le \frac{-\alpha (1-z)^2}{\beta^3 B_2}, z \in \left\{\frac{i}{k}, \frac{i+2}{k}\right\}
\end{align*} 

In our analysis, when we refer to the gap, this is what we are referencing.

\end{proof}

\subsection{Proof of Fact \ref{lem:derivative_approximation}} \label{appendix:analysis-stat-tilde}

\derivapproxlemma*
\begin{proof}
By the intermediate value theorem, there exists some point $y'\in [y, y+\Delta y]$, s.t.,
\[
g'(y') = \frac{g(y+\Delta y) - g(y)}{\Delta y}
\]
\begin{align*}
    |g'(y') - g'(y)| \le |g'(y + \Delta y) - g'(y)| \le B\cdot \Delta y
\end{align*}
\end{proof}
\subsection{Proof of \texorpdfstring{Lemma~\ref{cor:derivative-noisy-approximation}}{Section \ref{cor:derivative-noisy-approximation}}}
\corsecondderiv*
\begin{proof}
We apply Fact~\ref{lem:derivative_approximation} once to obtain $(B_1\Delta y_1)$-additive approximations $\tilde g'(y)$ and $\tilde g'(y+\Delta y)$.
These approximations are then used to compute:
\[
\tilde g''(y) = \frac{\tilde g'(y+\Delta y_2) - \tilde g'(y)}{\Delta y_2}
\]
which is a $2B_1 \frac{\Delta y_1}{\Delta y_2}$-additive approximation to the true $[g'(y+\Delta y_2) - g'(y)]/\Delta y_2$.
Now, we can apply Fact~\ref{lem:derivative_approximation} once again to obtain the final estimate,
which will have $2B_1 \frac{\Delta y_1}{\Delta y_2} + B_2\Delta y_2$ additive error.
\end{proof}

\section{Proof of Main Theorem} \label{sec:concentration_of_os}

\mainthm*
\begin{proof}

As we introduced earlier, we have the quantity $S$ that serves as a proxy for the tail weight of a distribution.
For a light-tailed distribution, the proxy quantity calculated at each bucket will lie above the threshold, whereas for an $(\alpha, \rho)$-heavy-tailed distribution, the proxy quantity will lie below the threshold for at least one of the buckets, where the threshold is dependent on $\alpha$, how quickly the hazard rate is decreasing. However, we cannot calculate $S$ directly without full knowledge of the distribution, so we consider a relaxation $\tilde S$, in which derivatives are approximated by difference quotients. Finally, we draw samples from the distribution to approximate $\tilde S$, and we call this approximation from samples $\hat S$ (Figure~\ref{fig:s_tildes_hats_errors_2}). We show that $\hat S$ requires the stated number of samples to distinguish between $(\alpha, \rho)$-heavy-tailed and light-tailed distributions. In order to prove the theorem, then, there are four steps:

\begin{enumerate}
    \item \textbf{Show that $\pmb{S}$ is a correct proxy for the tail weight.} 
    The proxy quantity stated as a condition on the tail of a distribution is equivalent to the condition on the tail of a distribution arising from the hazard rate definition (Theorem~\ref{thm:hr-test-statistic}). The gap implemented captures that the hazard rate must decrease by at least $\alpha$ in order for a distribution to be $(\alpha, \rho)$ heavy-tailed. See Section~\ref{sec:details_bucketing} and Appendix~\ref{appendix:analysis-stat-exact} for detailed discussion.
    \item \textbf{Show that $\pmb{\tilde S}$ is close to $\pmb{S}$ when there are ``enough'' buckets.} If we can estimate the numerator and the denominator of the proxy quantity accurately to an additive error, then we can estimate the proxy quantity correctly to an additive error if the proxy quantity is small; if it is big, then we know we are in the light-tailed region (Lemma~\ref{lem:s_tilde_error}). Indeed, we can accurately estimate the numerator and denominator accurately to an additive error while approximating derivatives with difference quotients due to the Lipschitzness (Fact \ref{lem:derivative_approximation}, Corollary \ref{cor:derivative-noisy-approximation}).
    \item \textbf{Show that $\pmb{\hat S}$ calculated from order statistics is close to $\pmb{\tilde S}$ when there are enough samples.} If we get a good multiplicative approximation of the lengths of the two buckets we are interested in, then we have a good multiplicative approximation to the statistic (Lemma~\ref{lem:bucket_to_stat_error}). We have multiplicative approximations for the lengths of the buckets because we can show that the order statistics used to estimate the endpoints concentrate well (Lemma~\ref{lem:additive_error_x}), which gives us an additive estimate for the error, and then we translate this to a multiplicative error (Lemma~\ref{lemma:samples}).
    \item \textbf{Determine how to limit the errors incurred in each step by explicitly calculating the number of buckets and number of samples.} In order to do this, we note that after incurring both the derivative approximation error and the sampling error, the test statistic $\hat S$ must still be on the ``correct'' side, that is, the same side of the threshold as $S$. The threshold is halfway between the lower bound for light-tailed distributions and the upper bound for $(\alpha, \rho)$-heavy-tailed distributions. 
    We ensure that we split the distribution into sufficiently many buckets that no more than a quarter of the gap is crossed due to approximation error. Further, we draw enough samples that no more a quarter of the gap is crossed due to sampling error. Thus, even when both errors are incurred, the test statistic remains on the correct side of the threshold.
\end{enumerate}


\begin{figure}
    \centering
    \vspace{5mm}
\tikzset{every picture/.style={line width=0.75pt}} 

\begin{tikzpicture}[x=0.75pt,y=0.75pt,yscale=-1,xscale=1]

\draw  [color={rgb, 255:red, 65; green, 117; blue, 5 }  ,draw opacity=1 ] (32.5,70.06) .. controls (32.5,63.44) and (37.87,58.06) .. (44.5,58.06) -- (80.5,58.06) .. controls (87.13,58.06) and (92.5,63.44) .. (92.5,70.06) -- (92.5,107.06) .. controls (92.5,113.69) and (87.13,119.06) .. (80.5,119.06) -- (44.5,119.06) .. controls (37.87,119.06) and (32.5,113.69) .. (32.5,107.06) -- cycle ;
\draw   (121,86.06) -- (246.39,86.06) -- (246.39,82.06) -- (261.5,90.03) -- (246.39,98) -- (246.39,94) -- (121,94) -- cycle ;
\draw  [color={rgb, 255:red, 65; green, 117; blue, 5 }  ,draw opacity=1 ] (294.5,70.06) .. controls (294.5,63.44) and (299.87,58.06) .. (306.5,58.06) -- (342.5,58.06) .. controls (349.13,58.06) and (354.5,63.44) .. (354.5,70.06) -- (354.5,107.06) .. controls (354.5,113.69) and (349.13,119.06) .. (342.5,119.06) -- (306.5,119.06) .. controls (299.87,119.06) and (294.5,113.69) .. (294.5,107.06) -- cycle ;
\draw   (391,86.06) -- (516.39,86.06) -- (516.39,82.06) -- (531.5,90.03) -- (516.39,98) -- (516.39,94) -- (391,94) -- cycle ;
\draw  [color={rgb, 255:red, 65; green, 117; blue, 5 }  ,draw opacity=1 ] (566.5,69.06) .. controls (566.5,62.44) and (571.87,57.06) .. (578.5,57.06) -- (614.5,57.06) .. controls (621.13,57.06) and (626.5,62.44) .. (626.5,69.06) -- (626.5,106.06) .. controls (626.5,112.69) and (621.13,118.06) .. (614.5,118.06) -- (578.5,118.06) .. controls (571.87,118.06) and (566.5,112.69) .. (566.5,106.06) -- cycle ;
\draw   (241.52,166.74) .. controls (241.55,162.07) and (239.24,159.72) .. (234.57,159.69) -- (198.81,159.44) .. controls (192.14,159.39) and (188.83,157.04) .. (188.86,152.37) .. controls (188.83,157.04) and (185.48,159.35) .. (178.81,159.3)(181.81,159.32) -- (143.05,159.05) .. controls (138.38,159.02) and (136.03,161.33) .. (136,166) ;
\draw    (187.17,135.06) -- (187.17,150.06) ;
\draw    (191.17,135.06) -- (191.17,150.06) ;
\draw   (183.18,138.46) .. controls (186.51,134.45) and (188.51,130.45) .. (189.18,126.46) .. controls (189.84,130.45) and (191.84,134.45) .. (195.18,138.46) ;
\draw   (513.52,166.74) .. controls (513.55,162.07) and (511.24,159.72) .. (506.57,159.69) -- (470.81,159.44) .. controls (464.14,159.39) and (460.83,157.04) .. (460.86,152.37) .. controls (460.83,157.04) and (457.48,159.35) .. (450.81,159.3)(453.81,159.32) -- (415.05,159.05) .. controls (410.38,159.02) and (408.03,161.33) .. (408,166) ;
\draw    (459.17,135.06) -- (459.17,150.06) ;
\draw    (463.17,135.06) -- (463.17,150.06) ;
\draw   (455.18,138.46) .. controls (458.51,134.45) and (460.51,130.45) .. (461.18,126.46) .. controls (461.84,130.45) and (463.84,134.45) .. (467.18,138.46) ;

\draw (48,81) node [anchor=north west][inner sep=0.75pt]   [align=left] {\begin{minipage}[lt]{20.113125pt}\setlength\topsep{0pt}
\begin{center}
$S$
\end{center}

\end{minipage}};
\draw (113,23) node [anchor=north west][inner sep=0.75pt]   [align=left] {\begin{minipage}[lt]{108.17312500000001pt}\setlength\topsep{0pt}
\begin{center}
approximate derivative \\by difference quotient
\end{center}

\end{minipage}};
\draw (146,104) node [anchor=north west][inner sep=0.75pt]   [align=left] {\begin{minipage}[lt]{56.96062500000001pt}\setlength\topsep{0pt}
\begin{center}
Lemma~\ref{lem:s_tilde_error}
\end{center}

\end{minipage}};
\draw (309,79) node [anchor=north west][inner sep=0.75pt]   [align=left] {\begin{minipage}[lt]{20.878125pt}\setlength\topsep{0pt}
\begin{center}
$\tilde{S}$
\end{center}

\end{minipage}};
\draw (366,23) node [anchor=north west][inner sep=0.75pt]   [align=left] {\begin{minipage}[lt]{133.673125pt}\setlength\topsep{0pt}
\begin{center}
calculate difference\\quotient from order statistics 
\end{center}

\end{minipage}};
\draw (438,70) node [anchor=north west][inner sep=0.75pt]   [align=left] {{\footnotesize \textcolor[rgb]{0.82,0.01,0.11}{error: $\displaystyle \epsilon _{2}$}}};
\draw (581,78) node [anchor=north west][inner sep=0.75pt]   [align=left] {\begin{minipage}[lt]{20.878125pt}\setlength\topsep{0pt}
\begin{center}
$\hat S$
\end{center}

\end{minipage}};
\draw (168,70) node [anchor=north west][inner sep=0.75pt]   [align=left] {{\footnotesize \textcolor[rgb]{0.82,0.01,0.11}{error: $\displaystyle \epsilon _{1}$}}};
\draw (108,174) node [anchor=north west][inner sep=0.75pt]   [align=left] {\begin{minipage}[lt]{39.578125pt}\setlength\topsep{0pt}
\begin{center}
Fact~\ref{lem:derivative_approximation}
\end{center}

\end{minipage}};
\draw (211,174) node [anchor=north west][inner sep=0.75pt]   [align=left] {\begin{minipage}[lt]{57.715pt}\setlength\topsep{0pt}
\begin{center}
Corollary~\ref{cor:bucket_bound_approximation}
\end{center}

\end{minipage}};
\draw (420,104) node [anchor=north west][inner sep=0.75pt]   [align=left] {\begin{minipage}[lt]{57.715pt}\setlength\topsep{0pt}
\begin{center}
Lemma~\ref{lem:bucket_to_stat_error}
\end{center}

\end{minipage}};
\draw (364,174) node [anchor=north west][inner sep=0.75pt]   [align=left] {\begin{minipage}[lt]{57.715pt}\setlength\topsep{0pt}
\begin{center}
Lemma~\ref{lem:additive_error_x}
\end{center}

\end{minipage}};
\draw (486,174) node [anchor=north west][inner sep=0.75pt]   [align=left] {\begin{minipage}[lt]{57.715pt}\setlength\topsep{0pt}
\begin{center}
Lemma~\ref{lemma:samples}
\end{center}

\end{minipage}};

\end{tikzpicture}
    \caption{The proxy quantity $S$ is a probe for the tail weight of the distribution based on the hazard rate. By approximating the derivatives in the proxy by difference quotients, we derive $\hat{S}$ while incurring error $\epsilon_1$. An analysis of how this error is incurred and what its value is given in the statements noted.}
    \label{fig:s_tildes_hats_errors_2}
\end{figure}

\subsection*{\textsc{Part 1/4: $\pmb{S}$ is an accurate proxy.}}
The proxy quantity derived in Theorem~\ref{thm:hr-test-statistic} considered with respect to the threshold gives a test which accurately determines whether a set of samples came from a light-tailed distribution or a distribution that is $(\alpha, \rho)$-heavy-tailed. An $(\alpha, \rho)$-heavy-tailed distribution has hazard rate decreasing at least by $\alpha$ over a region of the PDF with probability mass $\rho$, which gives us an expression for how far the statistic must be from the original threshold $1 - i/k$ in order for the hazard rate to be decreasing by at least $\alpha$. Further, if the region of mass $\rho$ lies in at least two buckets, then the proxy quantity will detect it. Thus, if a distribution is light-tailed, then \textit{all $k-3$} of the calculated proxy quantities calculated will lie above $1-i/k$; if even one of the proxies lies below $1 - i/k - \text{gap}$, then the distribution is $(\alpha, \rho)$-heavy-tailed. \par
More formal details for this can be found in Section~\ref{sec:details_bucketing} and Appendix~\ref{appendix:analysis-stat-exact}.

\subsection*{\texorpdfstring{\textsc{Part 2/4: $\pmb{\tilde S}$ is close to $\pmb{S}$.}}{\textsc{Part 2/4: ~S is close to S.}}}

In the next step, we show that approximating the derivatives in the proxy quantity by the respective difference quotients causes the test statistic to incur bounded error. In particular, we show how the error in the statistic is related to the number of buckets (which is directly related to how well we approximate derivatives; in the limit as $k \rightarrow \infty$, we recover the exact derivative). To this end, we first show that if we can estimate the numerator and the denominator of the proxy quantity to known additive errors, we can estimate the proxy quantity within an additive error, and we give the relationship between the two errors (Lemma~\ref{lem:s_tilde_error}). Then, to show that we can estimate the numerator and denominator of $\tilde S$ well, we show that the difference quotient approximation to the derivative incurs bounded error under stated Lipschitzness conditions (Fact \ref{lem:derivative_approximation}, Corollary \ref{cor:derivative-noisy-approximation}). \\

Recall that the proxy we are using is
\[
S = N/D = \left( \frac{d}{dx} F^{-1}(x)\right) / \left( \frac{d^2}{dx^2} F^{-1}(x) \right),
\]
which is being approximated by
\begin{align}
\tilde S = \tilde N / \tilde D ; \quad
\tilde N &\coloneqq \frac{F^{-1}\left( x + \frac{1}{k^2} \right) - F^{-1}\left( x\right)}{1/k^2} \\
\quad \quad 
\tilde D &\coloneqq \frac{\left(F^{-1}\left( x + \frac 1k + \frac{1}{k^2} \right) - F^{-1}\left( x + \frac{1}{k^2} \right) \right) - \left( F^{-1}\left( x + \frac{1}{k^2} \right) - F^{-1}\left( x  \right)\right)}{1/k^3}\,.
\end{align}

We show that if the error incurred in $\tilde N, \tilde D$ in approximating the derivatives by difference quotients has value $\epsilon'$, then we can estimate the value of the proxy quantity within 
an additive error of 
$\epsilon \coloneqq 6 \beta \epsilon'$, or the value of the proxy quantity is greater than 1.

\begin{framed}
\lemmaadditivestilde*

\begin{proof}
We start by noting that since $N = \frac{1}{f(F^{-1}(x))}$ and 
 the PDF is bounded by $\beta$,
 $N > 1/\beta = 6\epsilon'/\epsilon > 6\epsilon'$.
The last inequality is due to the assumption that $\epsilon < 1$. \par

First, assume that $S \geq \tilde S$. We consider the lowest possible value of $\tilde S = \tilde N/\tilde D$ and take the difference from the true value of $S$.
\begin{align}
 S - \tilde S \leq \frac ND - \frac{N-\epsilon'}{D+\epsilon'} = \frac{\epsilon' N + \epsilon' D}{D^2 + \epsilon' D}
\le \frac{\epsilon'(N + D)}{D^2} = \frac{\epsilon' S(S+1)}{N}\,.
\end{align}
We analyze the difference in two different cases:

\begin{itemize}
    \item \textbf{Case 1: $\pmb{S \leq 2}$\,.}
    In this case, the above difference is at most $\epsilon'\cdot 2 \cdot 3/N < \epsilon$
 since $N$ is at least ${6\epsilon'}/{\epsilon}$.
    \item \textbf{Case 2: $\pmb{S > 2}$\,.}
    In this case, we have $N/2 > D$ which implies $$\frac{N-\epsilon'}{D+\epsilon'} > \frac{N-\epsilon'}{N/2+\epsilon'} = 1 +  \frac{ N/2 - 2\epsilon'}{N/2 + \epsilon'} > 1\,.$$
The last inequality is due to the fact that $N > 6\epsilon'$.
\end{itemize}

Next, we assume that $S < \tilde S$. We consider the largest possible value of $\tilde S$ and take the difference from the true value of $S$ as before:
\begin{align}
\tilde S - S = \frac{N+\epsilon'}{D-\epsilon'} - \frac ND = \epsilon' \frac{N+D}{D^2 - \epsilon' D} = \epsilon' \frac{S+1}{D-\epsilon'}\,.
\end{align}
If $S > 1$, both $S$ and $\tilde S$ are larger than one, and we are done.
Otherwise, $D > N > 1/\beta$, which implies that $D-\epsilon' > {1}/{\beta} - {1}/(6\beta) > {1}/(2\beta) $,
and therefore, we can conclude that:
\[
\tilde S - S \le \epsilon' \frac{S+1}{D-\epsilon'} < 4\epsilon' \beta < 6\epsilon'\beta = \epsilon\,.
\]
Thus, the statement of the lemma is concluded. 
\end{proof}
\end{framed}

Now, we must show that we can estimate $N, D$ by $\tilde N, \tilde D$ respectively within additive error $\epsilon'$. First, we bound the additive error arising from the gradient approximation.
\derivapproxlemma*
As discussed in \ref{sec:proxy_buckets_discrete}, this allows us to determine the error bound for the difference quotient approximations of derivatives in the proxy quantity.
\corsecondderiv*
For reasons discussed in Section~\ref{sec:proxy_buckets_discrete}, we set $\Delta y_1 = 1/k^2$, and $\Delta y_2 = 1/k$,
meaning that the incurred error in the second derivative is $\frac{B_1 + B_2}{k}$. 
We set this to be less than $ \epsilon' = \frac{\epsilon}{6\beta}$ and find the value of $k$ for which this happens: $k > \frac{6\beta (2B_1 + B_2)}{\epsilon}$. Thus, we have the following corollary that specifies the number of buckets necessary to incur no more than $\epsilon$ error in estimating the proxy quantity when it is not greater than 1.
\begin{restatable}{corollary}{corDerivativeApprox}
\label{cor:bucket_bound_approximation}
If $k > 6\beta\frac{2B_1 + B_2}{\epsilon}$, and the estimate $\tilde S$ of the statistic $S$ is computed as described above, then either
both $S, \tilde{S} > 1$, or $|\tilde{S} - S| < \epsilon$.
\end{restatable}

\subsection*{\textsc{PART 3/4: $\pmb{\hat S}$ (calculated from order statistics) is close to $\pmb{\tilde S}$.}}


We must next show that we can approximate $\tilde S$ as defined in the previous section by the order statistics of a set of samples. First, we prove that if we can estimate the lengths of intervals accurately up to multiplicative error, we can estimate the test statistic accurately up to multiplicative error (Lemma~\ref{lem:bucket_to_stat_error}). Then, we show that we can achieve good multiplicative estimates of the lengths of the intervals using the following stages:
\begin{enumerate}
    \item We construct a map between order statistics of samples from a uniform distribution, denoted $Y_{(i)}$, (which we know concentrate well and in fact are \emph{sub-Gaussian}) and order statistics of samples from an arbitrary distribution (denoted $X_{(i)}$).
    \item Concentration of $Y_{(i)}$ implies concentration of $X_{(i)}$ up to accounting for the gap between $F(\mathbb{E}[X_{(i)}])$ and $\mathbb{E}[F(X_{(i)})]$ (Lemma~\ref{lem:additive_error_x}).
    \item The additive error between realization of order statistic and the quantity in the test statistic (which we get from the sub-Gaussianity parameter) can be converted to a multiplicative error in terms of the length of the bucket (Lemma~\ref{lemma:samples}).
\end{enumerate}



One can view $\Tilde{S}$ in terms of  two quantities:
\begin{align*}
\tilde L_1 &\coloneqq \frac{F^{-1}(y+1/k^2) - F^{-1}(y)}{1/k^2} \\ \\
\tilde L_2 &\coloneqq \frac{F^{-1}(y + 1/k + 1/k^2) - F^{-1}(y + 1/k)}{1/k^2} \, .
\end{align*}
In terms of the previous notation,
\begin{align*}
\tilde N = \tilde L_1 \quad &\quad \quad \tilde D = \frac{\tilde L_2 - \tilde L_1}{1/k} \, , \\  \\
\text{ giving us that: }
\tilde S &= \frac{\tilde L_1}{k(\tilde L_2 - \tilde L_1)} = \frac{\tilde N}{\tilde D}\, .
\end{align*}


In the following lemma, we show that if we estimate $\tilde L_1$ and $\tilde L_2$ accurately up to a multiplicative factor, and obtain $\hat{L}_1$ and $\hat{L}_2$, then $\tilde{S}$ can be approximated by
$$\hat{S} \coloneqq \frac{\hat{L}_1}{k \cdot \left(\hat{L}_2 - \hat{L}_1\right)}\,.$$

\clearpage

\begin{framed}
\lemmabuckettostaterror*
\begin{proof}
We analyze two separate cases. 
\paragraph{Case 1:  $\pmb{\tilde L_2 - \tilde L_1 \le \frac{1}{k+1} \tilde L_2}$.} Since $\tilde L_2 - \tilde L_1 > 0$ (due to montonicity of $f$), we have that $1 \ge \tilde L_1/\tilde L_2 \ge k/(k+1)$. Thus, we have that:
\begin{align}
\tilde{S} = \frac{\tilde L_1}{k(\tilde L_2 - \tilde L_1)} \ge \frac{\tilde L_1}{\frac{k}{k+1} \tilde L_2} \ge 1\,.
\end{align}
In turn, since we have $(1-\epsilon') \tilde L_1 \le \hat L_1 \le (1 + \epsilon') \tilde L_1\,$ and similarly for $\hat L_2$: 
\begin{align}
\hat{S} 
= \frac{\hat{L}_1}{k(\hat{L}_2 - \hat{L}_1)} 
&\ge \frac{(1-\epsilon') \tilde L_1}{k((1-\epsilon')(\tilde L_2 - \tilde L_1) + 2 \epsilon' \tilde L_2)} 
\ge \frac{(1-\epsilon')\frac{k}{k+1} \tilde L_2}{k (\frac{1 - \epsilon'}{k + 1} + 2 \epsilon') \tilde L_2} 
= \frac{1 - \epsilon'}{1 - \epsilon' + 2 \epsilon' (k + 1)} \\
&= 1 - \frac{2 \epsilon' (k + 1)}{1 - \epsilon' + 2 \epsilon' k}\,.
\end{align}
When $\epsilon' \le 1/(k^2 -k +1)$, $\hat{S} \ge 1 - 2/k$. Since $\tilde S > 1$ in this case, as well, the test statistic falls clearly in light-tailed territory.
 
This completes the first case. 
\paragraph{Case 2:  $\pmb{\tilde L_2 - \tilde L_1 > \frac{1}{k+1} \tilde L_2}\,,$} which can be rearranged to give $\tilde L_2 - \tilde L_1 \le \tilde L_2 < (k+1)(\tilde L_2 - \tilde L_1)$. We first show that $\hat S$ cannot be too small. In fact, we have:
\begin{align}
\hat{S} &\ge \frac{(1-\epsilon') \tilde L_1}{k((1 + \epsilon')\tilde L_2 - (1-\epsilon')\tilde L_1)} = \frac{(1-\epsilon') \tilde L_1}{k((1-\epsilon')(\tilde L_2 - \tilde L_1) + 2 \epsilon' \tilde L_2)} \\
&\ge \frac{(1-\epsilon') \tilde L_1}{k((1-\epsilon')(\tilde L_2 - \tilde L_1) + 2 \epsilon' (k+1) (\tilde L_2 - \tilde L_1))} \ge \frac{1-\epsilon'}{1 + \epsilon'(2k + 1)} \tilde{S} \ge (1-\epsilon)\tilde{S}\,.
\end{align}
where the last inequality holds if we set $\epsilon' \le \epsilon/(1 + (1-\epsilon)(2k + 1))\,.$ 
We can also find an upper bound for $\hat S$:
\begin{align}
\hat{S} &\le \frac{(1+\epsilon') \tilde L_1}{k\left((1 - \epsilon')\tilde L_2 - (1+\epsilon') \tilde L_1\right)} \le \frac{(1+\epsilon')\tilde L_1}{k\left((1 + \epsilon')(\tilde L_2 - \tilde L_1) - 2\epsilon' \tilde L_2\right)} \\
& \le \frac{(1+\epsilon')\tilde L_1}{k\left((1+\epsilon')(\tilde L_2 - \tilde L_1) - 2 \epsilon' (k+1)(\tilde L_2 - \tilde L_1)\right)} \le \frac{(1+\epsilon')}{1 - \epsilon'(2k + 1)} \tilde{S} \le (1 + \epsilon) \tilde{S} \,.
\end{align}

To satisfy this, we can set $\epsilon' \le \epsilon/(1 + (1+\epsilon)(2k + 1))\,.$ Note that this is always smaller than the condition for $\epsilon'$ derived using the lower bound.

Thus, by setting $\epsilon' =  \min \left(\frac{\epsilon}{1 + (1 + \epsilon)(2k + 1)}, \frac{1}{k^2} \right)$, we get that either we are in Case (1), where $\tilde{S} > 1, \hat{S} > 1 - 2/k$ or we can get an $1 \pm \epsilon$ multiplicative approximation.
\end{proof}
\end{framed}

Now that we know that we can accurately estimate $\tilde S$ by $\hat S$ if we have good estimates of the values of $\hat L_1, \hat L_2$, we will next show that we are able to estimate $\hat L_1, \hat L_2$ accurately, proceeding in 3 stages as discussed above.

\paragraph{1. Mapping to uniform distribution} To prove the concentration of the $X_{(i)}$s, we transform the samples from $f$ to samples from a uniform distribution over $[0,1]$ (Figure~\ref{fig:mapping}). The mapping is selected in such a way that it does not change the order of the elements. This fact implies that the order statistic with rank $i$, $X_{(i)}$ will be mapped to the order statistic with the same rank from the samples from the uniform distribution. Then, we show the concentration of order statistics according to the uniform distribution to establish the result.\par

\begin{figure}
    \centering
        \includegraphics[width=0.8\textwidth]{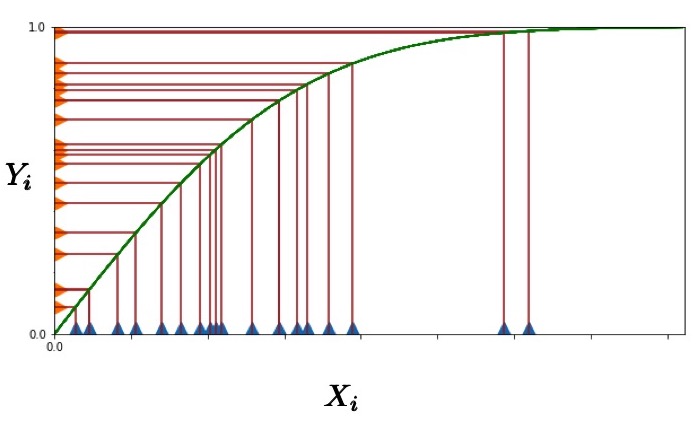} 
   
    \caption{In order to analyze the concentration of order statistic, we construct a map between samples drawn from a uniform distribution (orange points on vertical axis) and samples drawn from an arbitrary distribution (blue points on horizontal axis) with CDF $F$ (green curve). This map preserves several important properties as discussed in the text of the paper.
    }
    \label{fig:mapping}
\end{figure}

Formally, we map every sample $X$ to $Y = F(X)$. Since we assumed that $f$ is well-behaved, this mapping is a bijection. If we draw a random $X$ from $f$, it is not too hard to see that $Y = F(X)$ comes from a uniform distribution over $[0,1]$. In particular, for every $y \in [0,1]$, we have:
$$\Pr[Y]{Y \leq y} = \Pr[X]{X \leq F^{-1}(y)} = F(F^{-1}(y)) = y\,.$$
By the monotonicity of $F$, it is not hard to see that if we transform a set of $n$ samples from $f$, then $X_{(i)}$ will be mapped to the $i$-th element in the sorted list of the $Y$'s, denoted by $Y_{(i)}$. On the other hand, one can view $Y_{(i)}$ as the $i$-th order statistic among $n$ samples from a uniform distribution over $[0,1]$. In~\cite{marchal_sub-gaussianity_2017}, Marchal and Arbel have proved that $Y_{(i)} - \E{Y_{(i)}}$ is a sub-Gaussian random variable. 
We have used this fact to show the concentration of $X_{(i)}$'s around $F^{-1}(\E{Y_{(i)}}) = F^{-1}(i/(n+1))$. With appropriate choice $i$ and $n$, we can show the order statistic concentrates around the endpoint of the buckets as desired.

\paragraph{2. Concentration of $\pmb{X_{(i)}}$} Given this mapping, we can show that the order statistics from an arbitrary distribution concentration around $F^{-1}(i/(n+1))$. To do this, we use the sub-Gaussianity of $Y_{(i)}$ around $i/(n+1)$, transform this to a statement about $X_{(i)}$ related to $F^{-1}(i/(n+1))$, which is the endpoint of a bucket, and analyze terms separately. To this end, we present Lemma~\ref{lem:additive_error_x} and its proof.

\begin{figure}
    \centering
    \includegraphics[width=0.9 \textwidth]{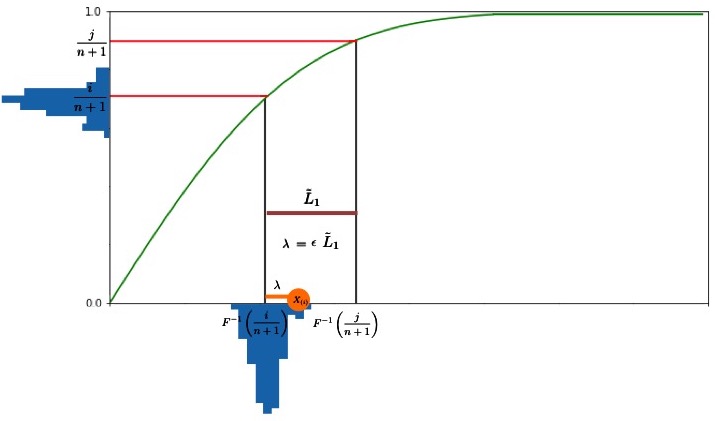}
    \caption{In Lemma~\ref{lem:additive_error_x}, we show that since $Y_{(i)}$ concentrates around $i/(n+1)$ (as evidenced by the blue histogram on the vertical axis), $X_{(i)}$ must also concentrate around $F^{-1}(i/(n+1))$ (as evidenced by the blue histogram on the horizontal axis). Then, in Lemma~\ref{lemma:samples}, we show that the distance that the order statistic (orange circle) falls from $F^{-1}(i/(n+1))$ (orange bar labeled $\lambda$) can converted to a multplicative error in terms of the length of the interval $\hat L_1$ (maroon bar).}
    \label{fig:multiplicative_error}
\end{figure}

\begin{restatable}{lemmma}{lemmaconcentrationx} \label{lem:additive_error_x}
Suppose $f$ is a well-behaved distribution. 
For $\epsilon \le \frac{B_1 f\left(F^{-1}\left(\frac{i}{n+1}\right)\right) + k}{k^2 f\left(F^{-1}\left(\frac{i}{n+1}\right)\right)}$, we have:
$$
\Pr{\left|X_{(i)} - F^{-1}\left(\frac{i}{n+1}\right)\right| \ge \epsilon} \le 2  \exp \left( -\epsilon^2 / \lambda^2 \right)\,, \quad \lambda \coloneqq \frac{B_1}{2 \,(n + 2) \, \left(\sqrt{4B_1 + 1} - 1\right) \, f(F^{-1}\left(\frac{i}{n+1}\right))}\,.
$$
\end{restatable}

\begin{proof}
As we explained earlier, we map the samples from $f$, $x_1, x_2, \ldots, x_n$, to $y_1, y_2, \ldots, y_n$ where $y_i = F(x_i)$. We have shown that one can view $y_i$'s as random samples from the uniform distribution over $[0,1]$. Since the mapping preserve the order, the $i$-th order statistic $X_{(i)}$ is mapped to the $i$-th order statistic among $y_i$'s, which we denote by $Y_{(i)}$. Now, we show the concentration of the order statistics from the uniform distribution. Note that the distribution over the order statistic is a beta distribution with parameters $\alpha \coloneqq i$ and $\beta \coloneqq n-i+1$:
$$\Pr{Y_{(i)} = y} = n \cdot \binom{n-1}{i-1} \cdot y^{i-1} \cdot (1-y)^{n - i}\,.$$
Marchal and Arbel have shown that a (centered) random variable drawn from the beta distribution is a sub-Gaussian random variable with parameter $\sigma^2 = 1/4(\alpha + \beta + 1) = 1/4(n+2)$ (Theorem 2.1 in~\cite{marchal_sub-gaussianity_2017}). In other words, for any $\lambda \in \mathbb{R}$, we have:

$$
\E{\exp\left(\lambda \cdot \left(Y_{(i)} - \E{Y_{(i)}}\right)\right)} \leq \exp \left(\frac{\lambda^2}{32\,(n+2)^2}\right)\,.
$$

Then, using equivalent definitions of sub-Gaussian random variables (See Proposition 2.5.2 in~\cite{vershynin18}.), we obtain:

$$
\Pr{ \left|Y_{(i)} - \E{Y_{(i)}}\right| \geq t} \leq 2\,\exp\left(
-16 \, (n + 2)^2 \,t^2
\right)\quad \quad \quad \forall t \geq 0\,.
$$

In the next step, using the Lipschitzness of $F^{-1}$, we relate the probability that $X$ deviates from $F^{-1}\left(\E{F\left(X_{(i)}\right)}\right)$ to the probability that $Y_{(i)}$ deviates from its expectation. Then, we use the sub-Gaussianity of $Y_{(i)}$'s to obtain the desired bound. 
Let $T$ denote $F^{-1}\left(\E{F\left(X_{(i)}\right)}\right) = F^{-1}\left(\E{Y_{(i)}}\right) = F^{-1}({i}/{n+1})$.
Our goal here is to find a bound of the following form for some parameter $\lambda$ which we determine later. 
$$\Pr{\left|X_{(i)} - F^{-1}\left(\frac{i}{n+1}\right) \right| \geq \epsilon} \leq 2\exp\left(-\epsilon^2/\lambda^2\right)\,.$$

Using the definition of the mapping, we have:

\begin{align}\label{eq:subG_X_i}
\Pr{\left|X_{(i)} - F^{-1}\left(\frac{i}{n+1}\right)\right| \geq \epsilon}
=
\Pr{\left|F^{-1}\left(Y_{(i)}\right) - F^{-1}\left(\E{Y_{(i)}}\right)\right| \geq \epsilon}
\end{align}

Using the mean value theorem, we know that there exists a $y$ between $Y_{(i)}$ and $\E{Y_{(i)}}$ such that: 
$$ {F^{-1}}'(y) \cdot \left(Y_{(i)} - \E{Y_{(i)}}\right) = F^{-1}\left(Y_{(i)}\right) - F^{-1}\left(\E{Y_{(i)}}\right)\,.$$
To bound the probability in \cref{eq:subG_X_i} via the above identity, we need to find an upper bound on ${F^{-1}}'(y)$. Note that since $F$ is concave, ${F^{-1}}'(y)$ is equal to $1/f(F^{-1}(y))$, and it is an increasing function.  Now, we consider a few different cases based on the interval that $y$ belongs to. 
More precisely, for some parameter $0 < \delta < \Theta(1/k)$ which we define later, we define the following events: 
\begin{itemize}
    \item $E_1 \coloneqq $ indicates the event where $Y_{(i)}  - \E{Y_{(i)}}  \leq 0$.
    \item $E_2 \coloneqq $ indicates the event where $0 < Y_{(i)}  - \E{Y_{(i)}} < \delta$.
    \item $E_3 \coloneqq $ indicates the event where $Y_{(i)}  - \E{Y_{(i)}} \geq \delta$.
\end{itemize}
Then, we rewrite the probability in \cref{eq:subG_X_i} in terms of conditional probabilities as follows:
\begin{equation}\label{eq:3pieces}
\begin{split}
\Pr{\left|X_{(i)} - F^{-1}\left(\frac{i}{n+1}\right)\right| \geq \epsilon}
& =
\Pr{\left|{F^{-1}}'(y) \cdot \left(Y_{(i)} - \E{Y_{(i)}}\right)\right| \geq \epsilon 
}
\\ &  =  
\Pr{\left.\left|{F^{-1}}'(y) \cdot \left(Y_{(i)} - \E{Y_{(i)}}\right)\right| \geq \epsilon \:\right\rvert\: E_1} \cdot \Pr{E_1}
\\ &  +  
\Pr{\left.\left|{F^{-1}}'(y) \cdot \left(Y_{(i)} - \E{Y_{(i)}}\right)\right| \geq \epsilon \:\right\rvert\: E_2} \cdot \Pr{E_2}
\\ &  +  
\Pr{\left.\left|{F^{-1}}'(y) \cdot \left(Y_{(i)} - \E{Y_{(i)}}\right)\right| \geq \epsilon \:\right\rvert\:  E_3 } \cdot \Pr{  E_3}\,.
\end{split}
\end{equation}
Below, we bound each of the above terms:
\begin{enumerate}
\item \textbf{First term:} In the case that $E_1$ holds,  $y$ is at most $\E{Y_{(i)}}$. Using the monotonicity of $f$, we have  ${F^{-1}}'(y)$ is at most $1/f\left(F^{-1}\left(\E{Y_{(i)}}\right)\right) = 1/f(F^{-1}\left({i}/{n+1}\right))$. This bound implies that
\begin{align*}
& \Pr{\left.\left|{F^{-1}}'(y) \cdot \left(Y_{(i)} - \E{Y_{(i)}}\right)\right| \geq \epsilon \:\right\rvert\: E_1} \cdot \Pr{ E_1}
\\ & \quad \quad \leq 
\Pr{\left. Y_{(i)} - \E{Y_{(i)}} \leq - \epsilon \cdot 
f\left(F^{-1}\left(\frac{i}{n+1}\right)\right)\:\right\rvert\: E_1}
\cdot \Pr{ E_1 }
\\ & \quad \quad =
\Pr{Y_{(i)} - \E{Y_{(i)}} \leq - \epsilon \cdot  f\left(F^{-1}\left(\frac{i}{n+1}\right)\right)}
\\ & \quad \quad - 
\Pr{\left. Y_{(i)} - \E{Y_{(i)}} \leq - \epsilon \cdot 
f\left(F^{-1}\left(\frac{i}{n+1}\right)\right)\:\right\rvert\: \overline{E_1} }
\cdot \Pr{ \overline{E_1} }\,,
\end{align*}

where the bar in $\overline{E_1}$ indicates the complement event of $E_{1}$. Note that $Y_{(i)} - \E{Y_{(i)}}$ cannot be smaller than a negative quantity if $E_1$ does not happen. Thus, the last term above is zero, and we have:
\begin{equation}
\label{eq:1stBound}
\begin{split}
    \Pr{\left.\left|{F^{-1}}'(y) \cdot \left(Y_{(i)} - \E{Y_{(i)}}\right)\right| \geq \epsilon \:\right\rvert\: E_1} \cdot & \Pr{ E_1} \\
&  \leq
\Pr{Y_{(i)} - \E{Y_{(i)}} \leq - \epsilon \cdot  f\left(F^{-1}\left(\frac{i}{n+1}\right)\right)}
    \end{split}
\end{equation}

\item \textbf{Second term:} In the case that $E_2$ holds, $y$ is at most $\E{Y_{(i)}}  + \delta$. 
Note that since $\delta $ is smaller than $1/k$, and we assume that the function is Lipschitz in the bucket which starts at $\E{Y_{(i)}}$, then we can use the Lipschitzness assumption as follows: 
$${F^{-1}}'(y) \leq {F^{-1}}'\left(\E{Y_{(i)} + \delta}\right) \leq 1/f\left(F^{-1}\left(\frac{i}{n+1}\right)\right) + \delta\,B_1\,.$$
Now, we have:
\begin{equation} 
\label{eq:2nBound}
\begin{split}
& \Pr{\left.\left|{F^{-1}}'(y) \cdot \left(Y_{(i)} - \E{Y_{(i)}}\right)\right| \geq \epsilon \:\right\rvert\: E_2} \cdot \Pr{E_2}
\\ & \quad \quad \leq
\Pr{\left.Y_{(i)} - \E{Y_{(i)}} \geq \frac{\epsilon}{\frac{1}{f\left(F^{-1}\left(\frac{i}{n+1}\right)\right)} + \delta B_1} \:\right\rvert\: E_2} \cdot \Pr{E_2}\,.
\end{split}
\end{equation}
We will use the above bound later and combine it with the bound for the third term.

\item \textbf{Third term:}  In the case where $E_3$ holds, we do not have an upper bound for ${F^{-1}}'(y)$. However, we exploit the fact that this case happens with a small probability.
\begin{equation} \label{eq:3rdBound}
\Pr{\left.\left|{F^{-1}}'(y) \cdot \left(Y_{(i)} - \E{Y_{(i)}}\right)\right| \geq \epsilon \:\right\rvert\: E_3} \cdot \Pr{E_3}
\leq
\Pr{E_3} = \Pr{Y_{(i)} - \E{Y_{(i)}} > \delta }\,.
\end{equation}

\item \textbf{Combining the bounds for the last two terms:}
Now we combine the two bounds above in \cref{eq:2nBound} and \cref{eq:3rdBound}. Note that if $\delta$ is smaller than $\epsilon/(1/f(F^{-1}\left(\frac{i}{n+1}\right)) + \delta B_1)$, then the right hand side in \cref{eq:2nBound} is equal to: 
$$
\Pr{\left.Y_{(i)} - \E{Y_{(i)}} \geq \frac{\epsilon}{\frac{1}{f\left(F^{-1}\left(\frac{i}{n+1}\right)\right)} + \delta B_1} \geq \delta \:\right\rvert\: 0 < Y_{(i)} - \E{Y_{(i)}} < \delta 
} \cdot \Pr{E_2} = 0\,.
$$
Thus, when delta is small, after combining the last two bounds, we get:
\begin{align*}
&\Pr{\left.\left|{F^{-1}}'(y) \cdot \left(Y_{(i)} - \E{Y_{(i)}}\right)\right| \geq \epsilon \:\right\rvert\: E_2} \cdot \Pr{E_2}
\\ & \quad \quad  +
\Pr{\left.\left|{F^{-1}}'(y) \cdot \left(Y_{(i)} - \E{Y_{(i)}}\right)\right| \geq \epsilon \:\right\rvert\: E_3} \cdot \Pr{E_3}
\\ & \quad \quad  \leq
\Pr{Y_{(i)} - \E{Y_{(i)}} \geq \delta}
\end{align*}
On the other hand, if $\delta$ is at least $\epsilon/(1/f(F^{-1}\left(\frac{i}{n+1}\right)) + \delta B_1)$, then
by combining \cref{eq:2nBound} and \cref{eq:3rdBound}, we have: 
\begin{align*}
&\Pr{\left.\left|{F^{-1}}'(y) \cdot \left(Y_{(i)} - \E{Y_{(i)}}\right)\right| \geq \epsilon \:\right\rvert\: E_2} \cdot \Pr{E_2}
\\ & \quad \quad  +
\Pr{\left.\left|{F^{-1}}'(y) \cdot \left(Y_{(i)} - \E{Y_{(i)}}\right)\right| \geq \epsilon \:\right\rvert\: E_3} \cdot \Pr{E_3}
\\ & \quad \quad  \leq
\Pr{\left.Y_{(i)} - \E{Y_{(i)}} \geq \frac{\epsilon}{\frac{1}{f(F^{-1}\left(\frac{i}{n+1}\right))} + \delta B_1} \:\right\rvert\: E_2} \cdot \Pr{E_2} + \Pr{E_3}
\\ & \quad \quad  \leq
\Pr{\left.Y_{(i)} - \E{Y_{(i)}} \geq \frac{\epsilon}{\frac{1}{f(F^{-1}\left(\frac{i}{n+1}\right))} + \delta B_1} \:\right\rvert\: E_2} \cdot \Pr{E_2} 
\\ & \quad \quad +
\underbrace{\Pr{\left.Y_{(i)} - \E{Y_{(i)}} \geq \frac{\epsilon}{\frac{1}{f(F^{-1}\left(\frac{i}{n+1}\right))} + \delta B_1} \:\right\rvert\: E_3}}_{=1}\cdot \Pr{E_3}
\\ & \quad \quad =
\Pr{\left.Y_{(i)} - \E{Y_{(i)}} \geq \frac{\epsilon}{\frac{1}{f(F^{-1}\left(\frac{i}{n+1}\right))} + \delta B_1} \:\right\rvert\: Y_{(i)} - \E{Y_{(i)}} > 0}\cdot \Pr{\E{Y_{(i)}} - Y_{(i)} < 0}
\\ & \quad \quad =
\Pr{Y_{(i)} - \E{Y_{(i)}} \geq \frac{\epsilon}{\frac{1}{f(F^{-1}\left(\frac{i}{n+1}\right))} + \delta B_1}}
\,.
\end{align*}
Note that the first probability in the last line above is exactly one, since we are conditioning on the fact that $Y_{(i)} - \E{Y_{(i)}}$ is at least $\delta$. Now, we get:

\begin{align*}
&\Pr{\left.\left|{F^{-1}}'(y) \cdot \left(Y_{(i)} - \E{Y_{(i)}}\right)\right| \geq \epsilon \:\right\rvert\: E_2} \cdot \Pr{E_2}
\\ & \quad \quad  +
\Pr{\left.\left|{F^{-1}}'(y) \cdot \left(Y_{(i)} - \E{Y_{(i)}}\right)\right| \geq \epsilon \:\right\rvert\: E_3} \cdot \Pr{E_3}
\\ & \quad \quad =
\Pr{\left.Y_{(i)} - \E{Y_{(i)}} \geq \frac{\epsilon}{\frac{1}{f(F^{-1}\left(\frac{i}{n+1}\right))} + \delta B_1} \:\right\rvert\: \E{Y_{(i)}} - Y_{(i)} < 0}\cdot \Pr{\E{Y_{(i)}} - Y_{(i)} < 0}
\\ & \quad \quad =
\Pr{Y_{(i)} - \E{Y_{(i)}} \geq \frac{\epsilon}{\frac{1}{f(F^{-1}\left(\frac{i}{n+1}\right))} + \delta B_1}}
\,.
\end{align*}

Now, to obtain the best bound we solve for $\delta = \epsilon/(1/f(F^{-1}\left(\frac{i}{n+1}\right)) + \delta B_1)$, and set $\delta$ to the positive solution of this quantity. Note that since $\epsilon$ is bounded from above, it is not hard to show that $\delta$ is at most $1/k$ as required in the beginning. 
Then, we achieve:
\begin{align*}
&\Pr{\left.\left|{F^{-1}}'(y) \cdot \left(Y_{(i)} - \E{Y_{(i)}}\right)\right| \geq \epsilon \:\right\rvert\: E_2} \cdot \Pr{E_2}
\\ & \quad \quad  +
\Pr{\left.\left|{F^{-1}}'(y) \cdot \left(Y_{(i)} - \E{Y_{(i)}}\right)\right| \geq \epsilon \:\right\rvert\: E_3} \cdot \Pr{E_3}
\\ & \quad \quad  \leq \Pr{Y_{(i)} - \E{Y_{(i)}} \geq \delta}
\\ & \quad \quad  \leq
\Pr{Y_{(i)} - \E{Y_{(i)}}  \geq \frac{\sqrt{4 \,f\left( F^{-1}\left(\frac{i}{n+1}\right) \right)^2 \,B_1 \,\epsilon +1} - 1}{2\,f\left(F^{-1}\left(\frac{i}{n+1}\right)\right)\,B_1}}
\\ & \quad \quad  \leq
\Pr{Y_{(i)} - \E{Y_{(i)}}  \geq \frac{\left(\sqrt{4 \, f\left(F^{-1}\left(\frac{i}{n+1}\right)\right)^2\, B_1 + 1} - 1\right)\cdot \epsilon}{2\,f\left(F^{-1}\left(\frac{i}{n+1}\right)\right) \,B_1} \geq \frac{\sqrt{4\,B_1 + 1} - 1}{2\,B_1} \cdot \epsilon \cdot f\left(F^{-1}\left(\frac{i}{n+1}\right)\right) 
}
\,.
\end{align*}
\end{enumerate}

Now, we put all  the pieces in \cref{eq:3pieces} back together. We combine the above bound with the bound we have obtained earlier in \cref{eq:1stBound} and get the following:
\begin{align*}
\Pr{\left|X_{(i)} - F^{-1}\left(\frac{i}{n+1}\right)\right| \geq \epsilon}
& \leq 
\Pr{Y_{(i)} - \E{Y_{(i)}} \leq -\, \epsilon \cdot  f\left(F^{-1}\left(\frac{i}{n+1}\right)\right)}
\\ &   + 
\Pr{Y_{(i)} - \E{Y_{(i)}}  \geq \frac{\sqrt{4\,B_1 + 1} - 1}{2\,B_1} \cdot \epsilon \cdot f\left(F^{-1}\left(\frac{i}{n+1}\right)\right)  }
\\ &   \leq
\Pr{\left|Y_{(i)} - \E{Y_{(i)}} \right| \geq \frac{\sqrt{4\,B_1 + 1} - 1}{2\,B_1} \cdot f\left(F^{-1}\left(\frac{i}{n+1}\right)\right) \cdot \epsilon }\,,
\end{align*}
where the last inequality is due to the fact that $\sqrt{4\,B_1 + 1 } - 1/(2 B_1)$ is always smaller than one. Now, we use the  sub-Gaussianity of $Y_{(i)} - \E{Y_{(i)}}$, and get

\begin{align*}
    \Pr{\left|X_{(i)} - F^{-1}\left(\frac{i}{n+1}\right)\right| \geq \epsilon}
& \leq
2\,\exp\left(- \frac{16\, (n+2)^2 \,\left(\sqrt{4\,B_1 + 1} - 1\right)^2 \,f\left(F^{-1}\left(\frac{i}{n+1}\right)\right)^2}{4\,B_1^2} \cdot  \epsilon^2\right) \,.
\\ 
& \leq 2 \exp\left(-\epsilon^2/\lambda^2\right)
\end{align*}
where the last line holds when we set the parameter $\lambda$ to be
\begin{equation} \label{eqn:subgauss_param_x-t}
\lambda \coloneqq \frac{B_1}{2 \,(n + 2) \, \left(\sqrt{4B_1 + 1} - 1\right) \, f\left(F^{-1}\left(\frac{i}{n+1}\right)\right)}\,.
\end{equation}
\end{proof}

\paragraph{3. Converting error in $X_{(i)}$ to error in statistic $\hat S$:} We can convert the additive error given in Lemma~\ref{lem:additive_error_x} into a multiplicative error in terms of the difference between $F^{-1}\left( j/(n+1) \right) - F^{-1}\left( i/(n+1) \right)$. This quantity represents the length of an interval in $\tilde{S}$. In Lemma~\ref{lemma:samples}, we show how to convert the additive error from Lemma~\ref{lem:additive_error_x} to a multiplicative error in terms of the length of the interval. 

\begin{restatable}{lemmma}{lemmasamplesconcentration} \label{lemma:samples}
For a sufficiently large parameter $\tau>1$ and number of buckets $k>2B_1\beta$, suppose we have $n = O(log(k) \cdot \sqrt{\sqrt{B_1} + 1} \cdot \tau) $ samples from a well-behaved distribution $f$. Then, we have: 
\begin{align*}
\Pr{\left|X_{(i)} - F^{-1}\left(\frac{i}{n+1}\right)\right| \geq \frac{2}{\tau (j-i)/(n+1)} \left(F^{-1}\left(\frac{j}{n+1}\right) - F^{-1}\left(\frac{i}{n+1}\right)\right) } \leq \frac{0.1}{4\,k^2}
\end{align*}
\end{restatable}

\begin{proof}
For a single interval, we have from Fact \ref{lem:derivative_approximation} that:
\begin{align}
\frac{F^{-1}\left(\frac{j}{n+1}\right)- F^{-1}\left(\frac{i}{n+1}\right)}{\frac{j-i}{n+1}} \ge \frac{1}{f\left(F^{-1}\left(\frac{i}{n+1}\right)\right) }- B_1 \frac{j-i}{n+1}
\end{align}
Note that this holds regardless of whether $j > i$ or $j < i$. We set $\frac{j-i}{n+1} \le \frac{1}{2 B_1 \beta}$, which gives us: $\frac{B_1}{k} \le \frac{1}{2\beta} \le \frac{1}{2f(F^{-1}\left(\frac{i}{n+1}\right))}$. Plugging this back in, we get:
\begin{align}
\frac{F^{-1}\left(\frac{j}{n+1}\right)- F^{-1}\left(\frac{i}{n+1}\right)}{\frac{j-i}{n+1}} &\ge \frac{1}{2 f(F^{-1}\left(\frac{i}{n+1}\right))} 
\end{align}

This implies then, that we can convert the additive bound to a multiplicative one. In particular, applying the bounds from Lemma~\ref{lem:additive_error_x}, this gives us that:
\begin{align}
\Pr{\left|X_{(i)} - F^{-1}\left(\frac{i}{n+1}\right)\right|\geq \frac{2}{\tau \frac{j-i}{n+1}} \left(F^{-1}\left(\frac{j}{n+1}\right) - F^{-1}\left(\frac{i}{n+1}\right)\right) \geq \frac{1}{\tau \cdot f(F^{-1}\left(\frac{i}{n+1}\right))} } \leq \frac{0.1}{4\,k^2}
\end{align}
\end{proof}

The above result shows that by union bound we can make sure all the order statistics we use are estimated with small error with high probability.

With this, we have all the pieces required to prove Theorem~\ref{thm:main}. We know that the proxy quantity is correct (Theorem~\ref{thm:hr-test-statistic}). We have shown that approximating the derivatives by difference quotients incurs bounded error (call this $\epsilon_1$) (Lemma~\ref{lem:s_tilde_error}), and we have now shown that we incur bounded error when using order statistics to approximate the test statistic (call this error $\epsilon_2$) (Lemma~\ref{lem:bucket_to_stat_error}). In the final section, we will show how many buckets and samples we require for the algorithm to succeed with high probability.

\subsection*{\textsc{Part 4/4: Errors can be set to satisfy theorem.}}
We start by giving a corollary that helps us translate the result of Lemma~\ref{lemma:samples} to one that is easy to use for analyzing the number of samples required.
\begin{corollary}
If we draw $n = O(\frac{k \log k}{\epsilon} \sqrt{\sqrt{B_1} + 1})$ samples, with probability at least 9/10, $\hat{L}$ concentrates around $\tilde{L}$ as defined previously claim. That is,
$$
(1-\epsilon) \cdot \tilde{L} \le \hat{L} \le (1+\epsilon) \cdot \tilde{L}\,.
$$
\end{corollary}

\begin{proof}
In order to prove this, we apply Lemma~\ref{lemma:samples}. As a corollary of that lemma, if $k \ge 2 B_1 \beta$ and $n = O(\log k \sqrt{\sqrt{B_1} + 1} \tau)\,,$ we have that:
\begin{align*}
&\quad \quad\Pr{}{\left|X_{(i (n+1)/k)} - F^{-1}\left(\frac{n+1}{k}\frac{i}{n+1}\right)\right| \geq \frac{2}{\tau \cdot 1/k} \left(F^{-1}\left(\frac{n+1}{k}\frac{i+1}{n+1}\right) - F^{-1}\left(\frac{n+1}{k}\frac{i}{n+1}\right)\right) } \leq \frac{0.1}{4\,k^2} \\
&\Leftrightarrow \Pr{}{\left|X_{(i (n+1)/k)} - F^{-1}\left(\frac{i}{k}\right)\right| \geq \frac{2}{\tau \cdot 1/k} \left(F^{-1}\left(\frac{i+1}{k}\right) - F^{-1}\left(\frac{i}{k}\right)\right) } \leq \frac{0.1}{4\,k^2} \\
\end{align*}

and 

\begin{align*}
&\quad \quad \Pr{}{\left|X_{((i+1) (n+1)/k)} - F^{-1}\left(\frac{i+1}{k}\right)\right| \geq \frac{2}{\tau \cdot -1/k} \left(F^{-1}\left(\frac{i}{k}\right) - F^{-1}\left(\frac{i+1}{k}\right)\right) } \leq \frac{0.1}{4\,k^2} \\
&\Leftrightarrow \Pr{}{\left|X_{((i+1) (n+1)/k)} - F^{-1}\left(\frac{i+1}{k}\right)\right| \geq \frac{2}{\tau \cdot 1/k} \left(F^{-1}\left(\frac{i+1}{k}\right) - F^{-1}\left(\frac{i}{k}\right)\right) } \leq \frac{0.1}{4\,k^2}
\end{align*}

Thus, with probability at least $1 - \frac{0.2}{4k^2}\,,$

\begin{align}
F^{-1}\left(\frac{i}{k}\right) - \frac{2}{\tau \cdot 1/k} \left(F^{-1}\left(\frac{i+1}{k}\right) - F^{-1}\left(\frac{i}{k}\right)\right) &\le X_{(i (n+1)/k)} \\
&\quad \quad \quad\le F^{-1}\left(\frac{i}{k}\right) + \frac{2}{\tau \cdot 1/k} \left(F^{-1}\left(\frac{i+1}{k}\right) - F^{-1}\left(\frac{i}{k}\right)\right) \\
F^{-1}\left(\frac{i+1}{k}\right) - \frac{2}{\tau \cdot 1/k} \left(F^{-1}\left(\frac{i+1}{k}\right) - F^{-1}\left(\frac{i}{k}\right)\right) &\le X_{((i+1) (n+1)/k)} \\
&\quad \quad \quad\le F^{-1}\left(\frac{i+1}{k}\right) + \frac{2}{\tau \cdot 1/k} \left(F^{-1}\left(\frac{i+1}{k}\right) - F^{-1}\left(\frac{i}{k}\right)\right)\,.
\end{align}

And so with probability at least $1-\frac{0.2}{4k^2}$:
\begin{align}
X_{((i+1) (n+1)/k)} - X_{(i (n+1)/k)} &\le F^{-1}\left(\frac{i+1}{k}\right) + \frac{2}{\tau \cdot 1/k} \left(F^{-1}\left(\frac{i+1}{k}\right) - F^{-1}\left(\frac{i}{k}\right)\right) \\
&\quad \quad \quad- \left( F^{-1}\left(\frac{i}{k}\right) - \frac{2}{\tau \cdot 1/k} \left(F^{-1}\left(\frac{i+1}{k}\right) - F^{-1}\left(\frac{i}{k}\right)\right) \right) \\
&= \left( 1 + \frac{4k}{\tau}   \right) \left(F^{-1}\left(\frac{i+1}{k}\right) - F^{-1}\left(\frac{i}{k}\right)\right) 
\end{align}
and 
\begin{align}
X_{((i+1) (n+1)/k)} - X_{(i (n+1)/k)} &\ge F^{-1}\left(\frac{i+1}{k}\right) - \frac{2}{\tau \cdot 1/k} \left(F^{-1}\left(\frac{i+1}{k}\right) - F^{-1}\left(\frac{i}{k}\right)\right) \\
&\quad \quad \quad - \left( F^{-1}\left(\frac{i}{k}\right) + \frac{2}{\tau \cdot 1/k} \left(F^{-1}\left(\frac{i+1}{k}\right) - F^{-1}\left(\frac{i}{k}\right)\right) \right) \\
&= \left( 1 - \frac{4k}{\tau}   \right) \left(F^{-1}\left(\frac{i+1}{k}\right) - F^{-1}\left(\frac{i}{k}\right)\right) 
\end{align}

This tells us that if we draw $n = O(\frac{k \log k}{\epsilon} \sqrt{\sqrt{B_1} + 1} )$ samples, then with probability at least $1 - \frac{0.2}{4k^2}$:
$$
\left( 1 - \epsilon  \right) \tilde{L} \le \hat{L} \le \left( 1 + \epsilon   \right) \tilde{L} \,.
$$

\end{proof}

Now, we analyze exactly how much error we can incur while still remaining on the correct side of the threshold. In particular, given that both approximating the derivative and sampling introduce errors, we seek to draw enough samples that the test statistic calculated from samples lies on the same side of the threshold as the proxy quantity. We have essentially two sources of error that we can control: the first is $\epsilon_1$, the additive error from the derivative approximation; the second is $\epsilon_2$, the multiplicative error from sampling. Keeping $\epsilon_1$ and $\epsilon_2$ small enough to not cross the threshold with high probability entails using enough buckets and sufficiently many samples per bucket.
Based on the earlier derivation of the absolute value of the gap, we have that:
\begin{align}
\bigg \lvert \frac{\alpha (1-z_2)^2}{f'(F^{-1}(z_2))} \bigg \rvert  &\ge \frac{\alpha(1-z_2)^2}{f(F^{-1}(z_2))^3 B_1} \quad \quad \quad  \text{ because }  \quad |F^{-1''}(y)| = \bigg \lvert  \frac{-f'(F^{-1}(x))}{f(F^{-1}(x))^3}\bigg \rvert \le B_1
\end{align}

Let us define $\gamma_i \coloneqq \frac{\alpha (1-i/k)^2}{f(F^{-1}(i/k))^3 \, B_1} \ge \gamma \coloneqq \frac{\alpha \rho^2}{\beta^3 B_1}\,,$ a lower bound on the actual gap. We use this quantity in our analysis. Thus, the lower bound of the test statistic calculated for a light-tailed distribution is:
\begin{align*}
(1- \epsilon_2)(S-\epsilon_1) = S - \epsilon_1 - \epsilon_2 \, S + \epsilon_1 \epsilon_2 \ge 1- \frac{i}{k} - \frac{1}{10} \gamma
\end{align*}

First, we know that for a light-tailed distribution, $S \ge 1- \frac1k\,,$ so we require that $\epsilon_1 \epsilon_2 - \epsilon_1 - \epsilon_2 \, S \ge \frac{-\gamma}{10}\,.$ Let us first consider the case where $ S \le 1\,,$ Then, if $\epsilon_1 \le \gamma/20$ and $\epsilon_2 \le \gamma/20\,,$ we have that:
$$
-\frac{\gamma}{20} - \frac{\gamma}{20} \cdot S + \frac{\gamma^2}{20} \ge  -\frac{\gamma}{20} - \frac{\gamma}{20} = -\frac{\gamma}{10}\,.
$$

This tells us that it suffices to have $\epsilon_1 = \gamma/20$ and $\epsilon_2 = \gamma/20\,.$ Thus, we select the number of buckets and the number of samples in such a way as to achieve $\epsilon_1, \epsilon_2 \le \frac{\gamma}{20}\,.$ In particular, from Corollary~\ref{cor:bucket_bound_approximation}, we know that if $k \ge \Theta(\beta(2B_1 + B_2)/\gamma)\,,$ then $|\tilde{S} - S| < \gamma/20\,.$ Next, Lemma~\ref{lem:bucket_to_stat_error} tells us that in this case, we need estimates of $\tilde{L}$ to a multiplicative factor of $1 \pm \epsilon'\,,$ where $\epsilon' =    \Theta( \gamma/((20 + \gamma)k))$. From the claim above, we have that with $n =\Theta\left(\frac{(k^2 \log k) \sqrt{\sqrt{B_1} + 1}}{\gamma}\right)\,,$ we achieve this guarantee.

Next, if we are still in the light-tailed case and $S \ge 1\,,$ then regardless of $\epsilon_2\,,$ Lemma~\ref{lem:bucket_to_stat_error} tells us we require estimates of $\tilde{L}$ to $1 \pm \epsilon'\,,$ where $\epsilon' = \Theta(1/k^2)\,.$ Again, from the claim above, we have that with $n = \Theta\left( k^3 \log k \sqrt{\sqrt{B_1} + 1}  \right)\,,$ we can achieve the desired guarantee. \par

Having addressed the light-tailed case, let us consider the $(\alpha, \rho)-$heavy-tailed case. We have that:
$$
\hat{S} \le (1+ \epsilon_2) (S + \epsilon_1)\,, \text{which we want to be} \le 1 - \frac ik -\frac{9\gamma}{10}\,.
$$

Since the distribution is $\alpha, \rho$-heavy tailed, we know that $S \le 1 - \frac ik - \gamma\,.$ Thus, we set $\epsilon_1$ such that $S + \epsilon_1 \le 1 - \frac{i}{k} - \frac{19 \gamma}{20} = 1 - \frac ik - \gamma + \frac{\gamma}{20}\,,$ i.e., we set $\epsilon_1 = \gamma/20\,.$ Now, $\epsilon_2 = \gamma/20$ suffices, since:
$$
\left( 1 + \frac{\gamma}{20} \right) \left ( 1 - \frac{i}{k} - \frac{19 \gamma}{20} \right ) \le 1 - \frac{i}{k} + \frac{\gamma}{20} \left(  1 - \frac ik  \right) - \left( 1 + \frac{\gamma}{20} \right) \frac{19 \gamma}{20} \le 1 - \frac ik + \frac{\gamma}{20} - \frac{19 \gamma}{20} = 1 - \frac ik  - \frac{9 \gamma}{10}
$$
as desired. We apply the union bound over the $k^2$ order statistics we compute to get that all simultaneously concentrate to $1 \pm \epsilon_2$ multiplicative error with probability $\ge 0.1$. Thus, the number of samples required is $n = \Theta \left(  \frac{k^2 \log k \sqrt{\sqrt{B_1} + 1}}{\gamma}  \right),$ and the number of required buckets is $k \ge \Theta( \frac{\beta (2B_1 + B_2)}{\gamma})\,.$

In the worst case, $\gamma \ge \Theta( \alpha \rho^2 \beta^{-3} B_1^{-1})\,,$ and so $k \ge \frac{\beta^4 B_1 (2 B_1 + B_2)}{\alpha \, \rho^2}\,.$ Then, 
$$ 
n = \max \left \{ \frac{ \beta^3 B_1 k^2 \log k \sqrt{\sqrt{B_1} + 1}}{\alpha\, \rho^2}, k^3 \log k \sqrt{\sqrt{B_1} + 1}   \right\}\,.
$$
Equivalently, 
$$
n = \max \left\{ \frac{\beta^3 B_1}{\alpha \rho^2}  , k \right\} \cdot k^2 \log k \sqrt{\sqrt{B_1} + 1}  \,.
$$

\end{proof}

\section{Hardness Result}
\label{sec:lowerbound}

In this section, we show for any number of samples $m$, we can construct two classes of distributions, one light-tailed and the other heavy-tailed, such that they are indistinguishable using $m$ samples. Our lower bound implies that there is no algorithm to distinguish the class of light-tailed and heavy-tailed distributions unless further assumptions is made. Here we give an alternative definition of heavy tails, which does not require the heavy-tailed part be contiguous. Then we show that it is hard to distinguish these type of heavy-tailed distributions from the light-tailed ones.

\begin{definition}
We say a distribution $p$ is  $(\alpha, \rho)$-scattered-heavy-tailed if the hazard rate is decreasing by rate at least $\alpha$ over measurable intervals of the domain with probability mass at least $\rho$. 
\end{definition}

\paragraph{High level idea:} We construct two classes of distributions that are hard to distinguish with few samples: light tailed distributions, $\CC_L$, and heavy-tailed distribution $\CC_H$.  The class of light-tailed distributions contains only one member that is an exponential distribution with $\fE(x) = e^{-x}$. We construct the class of heavy-tailed distributions via a randomized process as follows: We start off by the same distribution $\fE(x) = e^{-x}$. We split the domain of $\fE$ into $s$ {\em chunks} such that the probability mass in every chunk is equal to $1/s$. Then, we select roughly $\rho' = \Theta(\rho)$ fraction of these chunks randomly and embed a heavy-tailed distribution, namely $\fH$, in (some of) those selected chunks. The construction has two key properties: First, the probability mass of a chunk remains the same even when the alteration happens. Second, if we draw one sample from a chunk, we cannot tell whether it is altered or not. The key idea which leads to this property is in the embedding. When we alter a chunk, we randomly replace $\fE$ by a heavy-tailed piece $\fH$ or another partial PDF $\fbar$. We simply define $\fbar$ such that the mixture of $\fH$ and $\fbar$ each with probability a half gives us exactly $\fE$. Thus, if we receive one sample from a chunk that comes from a random $\CC_H$, it is impossible to tell whether the chunk is altered or not. It is worth noting that this process generates a class of distributions, $\CC_H$, that depends on a parameter $s$. We may also $\CC_{H}(s)$ to denote it. To complete our proof, we show that for any algorithm that uses $m$ samples, there is a sufficiently large $s$ such that it is very unlikely to have more than one sample per chunk. Thus, $\CC_L$ and $\CC_{H}(s)$ are indistinguishable when we use $m$ samples. This fact implies that no algorithm that uses finitely many samples can distinguish a light tailed distribution from a distribution that is heavy on measurable subset of the domain with mass $\rho$ unless we make further assumptions including that the heavy-tailed part might need to be contiguous.

\begin{theorem}
For any integer $m$, there is no algorithm that receives $m$ samples from $p\,,$ a monotone and continuous distribution, and can distinguish whether $p$ is light-tailed distribution or a $(\alpha, \rho)-$scattered-heavy-tailed for $\alpha < 0.0043$ and $\rho < 0.5$ with probability more than $0.5 + o(1)$.
\end{theorem}

\begin{proof} \
By way of contradiction, let us assume there exists such algorithm, namely $\AA$, which uses $m$ samples and can distinguish if light-tailed distributions from heavy-tailed ones with probability at least $0.5 + \delta$. Here, we assume that $\delta$  is in $(0,1/2]$ and $1/\delta$ is a constant. (i.e., $\delta$ is sufficiently away from zero.) To prove our theorem: we construct two classes of distributions: $\CC_L$ and $\CC_H(s)$ for a sufficiently large parameter $s$. The main goal is to show that if we draw $m$ samples from the distribution in each of these classes, the outcome should be very similar, thus no algorithm can tell the difference better than guessing. 

\paragraph{Construction of classes:} As we mentioned earlier, $\CC_L$ contains a single exponential distribution with the pdf $\fE(x) = e^{-x}$. We construct the class of heavy-tailed distribution via a randomized process as follows: We start off with the same distribution $\fE(x) = e^{-x}$. We split the domain of $\fE$ into $s$ {\em chunks} such that the probability mass in every chunk is equal to $1/s$. Then, for $i \leq 3s/4$, we select chunk $i$ with probability $\rho' = 2\rho$ randomly and embed a heavy-tailed distribution, namely $\fH$, in those selected chunks. When we alter a chunk, we randomly replace $\fE$ by a heavy-tailed piece $\fH$ or another partial PDF $\fbar$. We simply define $\fbar$ such that the mixture of $\fH$ and $\fbar$ each with probability a half gives us exactly $\fE$. $\CC_H$ consists of the distributions that result from this corruption process. Below are the details of a single alteration of the exponential in a chunk by $f_H$. The alteration on one chunk via $f_H$ has the following properties:

\begin{itemize}

\item The interval has the same mass that some exponential distribution (light-tailed) would have on that interval (i.e., $1/s$), but the hazard rate is decreasing on most of the mass of that interval. We call this the \emph{fooling} region of the interval.

\item In place of the exponential distribution on that interval, we use a fast-dropping exponential, followed by a uniform, followed by another fast-dropping exponential (see Figure~\ref{fig:corrupted}).

\item Note that once the starting point and the amount of weight of the interval are fixed, for a fixed parameter $\beta \in (1, 2)$ denoting the drop rate of the fast-dropping exponential, the rest of the construction is determined. We fix this parameter to $\beta = 1.5\,,$ but a different choice would lead to a similar analysis with slightly different constants.

\item Note that since we are choosing $i \leq (3/4)s$, $x_B$ for the rightmost chunk we could select is $\ln(4)\,.$
\end{itemize}

\begin{figure}
    \centering
    \includegraphics[width=0.75\textwidth]{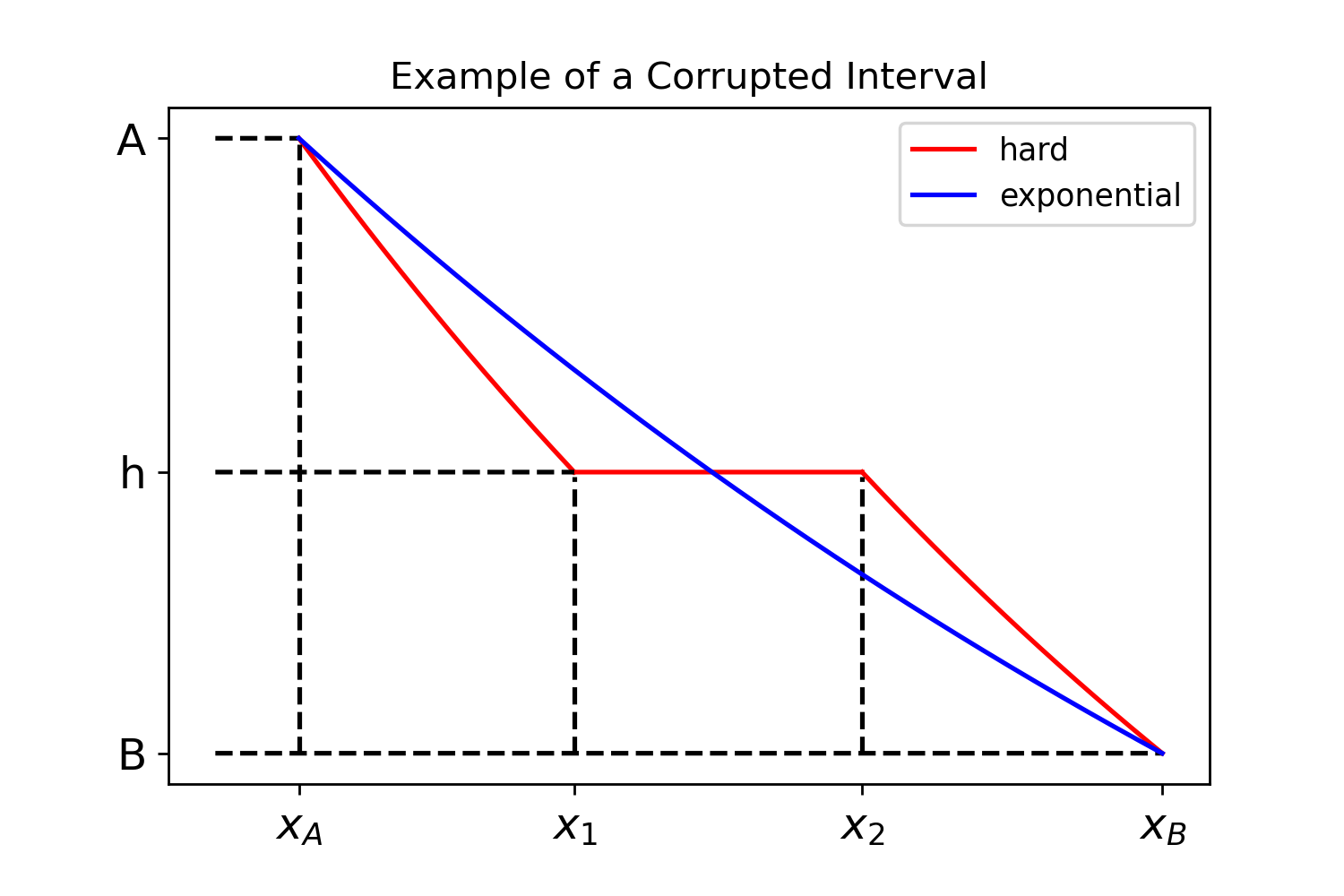}
    \caption{An example of a corrupted interval.}
    \label{fig:corrupted}
\end{figure}

\paragraph{Construction of $f_H$:} We start off by a chunk of weight $C \coloneqq 1/s$ starting at $x = x_A$ and the distribution $\fE(x) = e^{-x}\,.$ Note that this uniquely defines the end point of the chunk $x_B$. Clearly, we have: 
$$\fE(x_A) = A\,, \quad \quad \fE(x_B) = B\,,\mbox{ and }\quad \quad \int \limits_{x_A}^{x_B} \fE(x) dx = A - B = C\,.$$

Letting $x_1\in (x_A, x_B)$ denote the point at which the distribution switches from exponential to uniform and $x_2 \in (x_A, x_B)$ denote the point where the distribution switches back from uniform to exponential in the altered chunk, $f_H$ is given by:
\begin{align}
f_\text{H}(x) = \begin{cases}
Ae^{-\beta (x-x_A)} & x \in [x_A, x_1) \\
h & x \in [x_1, x_2) \\
Be^{-\beta (x - x_B)} & x \in [x_2, x_B)
\end{cases}
\end{align}

From the fact that this construction is continuous, we have that $h = Ae^{-\beta(x_1-x_A)} = Be^{-\beta(x_2-x_B)}\,.$ We use that to solve for $x_1, x_2$:
\begin{align}
x_1 &= x_A + \frac{1}{\beta} \ln \left( \frac{A}{h}\right) \\
x_2 &= x_B + \frac{1}{\beta} \ln \left( \frac{B}{h} \right) \\
&= \frac{-1}{\beta} \ln \left( \frac{A}{B}\right) + x_1 - x_A + x_B \\
&\Rightarrow x_2 - x_1 = \frac{\beta - 1}{\beta} (x_B - x_A)\,.
\end{align}

We use $\FE$ and $\FH$ to refer to the CDF of $\fE$ and $\fH$ as well. We get that the CDF is given by the following:

\begin{align}
F_\text{hard}(x) = \begin{cases}
F_\text{exp}(x_A) + \frac{A}{\beta}\left(1 - e^{-\beta(x-x_A)}\right)& x \in [x_A, x_1) \\
F_\text{exp}(x_A) + \frac{A}{\beta}\left(1 - e^{-\beta(x_1-x_A)}\right) + h(x - x_1)& x \in [x_1, x_2) \\
F_\text{exp}(x_A) + \frac{A}{\beta}\left(1 - e^{-\beta(x_1-x_A)}\right) + h(x_2 - x_1) + \frac{B}{\beta}(e^{-\beta (x-x_B)} - 1) & x \in [x_2, x_B)
\end{cases}
\end{align}

Knowing that $F_\text{hard}(x_B) - F_\text{hard}(x_A) = C\,,$ we enforce the area constraint to solve for $h$:
\begin{align}
\frac{A}{\beta} \left( 1 - e^{-\beta (x_1 - x_A)}\right) + h(x_2 - x_1) + \frac{B}{\beta} \left(e^{-\beta(x_2 - x_B)} - 1\right) &= C \\
\frac{A}{ \beta} \left( 1 - \frac{h}{A}\right) + h \cdot \frac{\beta -1}{\beta} (x_B - x_A) + \frac{B}{\beta} \left(\frac{h}{B} - 1 \right) &= C \\
\frac{h}{\beta} \left( -1 + (\beta - 1)(x_B - x_A) + 1\right) &= C - \frac{A}{\beta} + \frac{B}{\beta} \\
h = \frac{\beta C - A + B}{(\beta - 1)(x_B - x_A)} &= \frac{C}{(x_B-x_A)}
\end{align}

\paragraph{Analysis of $\fbar$: } Recall that we simply set $\fbar(x) = 2\fE(x) - \fH(x)$. We need to show that on the region of interest, $\fbar(x)$ is never below zero, and it is monotone decreasing. Given that, it is clear that the probability mass of a chunk when we replace $\fE$ by $\fbar$ remains equal to $C$ as required.

\textit{Existence:} To show that such an $\fbar$ exists, we must show that $2e^{-x} \ge f_\text{hard}(x)\,.$ In the first region, this is clearly true, as $2 e^{-x} \ge e^{-x} = A e^{-(x-x_A)} \ge A e^{-\beta (x-x_A)}$ when $\beta \ge 1\,.$ In the uniform region, we compare $2e^{-x}$ to a constant function, so it suffices to compare the smallest value of $2e^{-x}$ to the constant height:

\begin{align}
2e^{-x} \ge h &\Leftrightarrow \ln\left( \frac{2}{h}\right) \ge x \\
\ln \left( \frac{2}{h} \right) &\ge x_2 = x_B + \frac{1}{\beta} \ln \frac{B}{h} = \ln\left(  \frac{1}{B^{1 - 1/\beta} h^{1/\beta}}  \right) \\
&\Leftrightarrow 2 B^{1 - 1/\beta} \ge h^{1-1/\beta} \label{eqn:f-bar-cond}
  \end{align}
We return to this condition in a moment after considering the second exponential region, where we want to determine when: $2e^{-x} \ge Be^{-\beta(x-x_B)} \,, x \in [x_2, x_B)$. This is equivalent to checking where $2B^{\beta - 1} \ge e^{-x(\beta-1)}$ holds. For this to hold, it is sufficient for $2B^{\beta - 1} \ge 1\,,$ meaning that $B \ge (1/2)^{1/(\beta - 1)}\,.$ Note that if this holds, then Eqn.~\ref{eqn:f-bar-cond} also holds. When $B \leq e^{-\ln(4)}\,, \beta = 1.5\,,$ this condition indeed holds, which means Eqn.~\ref{eqn:f-bar-cond} also holds. \par

\textit{Monotonicity:} To show montonicity, we must show that $\bar f_H(x) = 2e^{-x} - f_H(x)$ is monotone decreasing as well in the regions in which we are interested. We calculate the derivative:
\begin{align}
    f'_H(x) &= \begin{cases}
    - \beta A e^{-\beta (x-x_A)} & x \in [x_A, x_1) \\
    0 & x \in (x_1, x_2) \\
    -\beta B e^{-\beta(x-x_B)} & x \in (x_2, x_B]
    \end{cases}\\
    \Rightarrow \text{ we must check } \bar f_H'(x) &= -2e^{-x} - f_H'(x) = -2e^{-x} + \beta A e^{-\beta(x-x_A)} \le 0\,, \\
    \bar f_H'(x) &= -2e^{-x} - 0 \le 0 \,, \\
    \bar f_H'(x) &= -2e^{-x} - f_H'(x) = -2e^{-x} + \beta B e^{-\beta(x-x_B)} \le 0\,,
\end{align}
where the second condition clearly holds, and we can summarize the first and third as checking when $-2e^{-x} + \beta D e^{-\beta(x-x_D)} \le 0$ holds, where $D = A, x_D = x_A$ or $D = B, x_D = x_B\,.$ Then:
\begin{align}
    &\quad \quad 2e^{-x} \ge \beta D e^{-\beta(x-x_D)} \\
    &\Leftrightarrow e^{-\beta(x-x_D)} \le \frac{2}{\beta} e^{-(x-x_D)} \Leftrightarrow e^{-(\beta-1)(x-x_D)} \le \frac{2}{\beta}
    \Leftrightarrow -(\beta - 1) (x-x_D) \le \ln \left(\frac{2}{\beta} \right) \\
    &\Leftarrow \frac{\ln \frac{2}{\beta}}{\beta - 1} \ge |x-x_D| \Leftarrow \frac{\ln \frac{2}{\beta}}{\beta - 1} \ge (x_B - x_A)\,. \label{eqn:final-monfbar}
\end{align}

Since we have that $x_B \le \ln(4)\,,$ and $\beta = 1.5\,,$ still consider chunks in a constant fraction of the mass of the distribution, so we have that $x_B - x_A \le \ln(1 + 4C)\,.$ Thus, we are interested in the region where the left hand side of Eqn.~\ref{eqn:final-monfbar} is greater than $1 + 4C\,.$ We have full freedom to pick $\beta$ subject to $2 > \beta > 1\,.$ If we set $\beta = 1.5\,,$ then this condition holds when $C \le 0.194\,,$ (i.e., $s \ge 5.2)$ which is satisfied. 
Thus, we can conclude the existence and monotonicity of $\fbar$.

\paragraph{Drawing samples from classes:} Here, we explain a slightly modified process, that is called  {\em poissonization method}, to draw roughly $m$ samples from the two classes of distribution. 
Suppose we generate samples from $\CC \in \{\CC_L, \CC_H\}$. Let $p$ be a random distribution in $\CC$. 
 That is, instead of drawing $m$ samples from $p$, we draw $\poi(m')$ samples from $p$ where $m' = \Theta(m)$. In this way, the number of samples we see in every chunk is independent from the rest which simplifies the proof greatly.  Let $X$ denotes the sample set we obtain according to this process. We use $\DD_L$ and $\DD_H$ to denote the distributions of $X$ when $\CC = \CC_L$ and $\CC = \CC_H$ respectively.

\paragraph{Proof of indistinguishability:} Suppose we set $\CC$ to be $\CC_L$ and $\CC_H$ each with probability a half. Then we draw a sample from $X$ from the class $\CC$ as explained above. We define event $\EE$ to be the probabilistic event when $X$ contains at least $m$ samples and $p$ is a light-tailed distribution, or an $(\alpha, \rho)$-scattered-heavy tailed one. It is worth noting that condition on $\EE$, then algorithm $\AA$ distinguishes whether $\CC = \CC_L$ or $\CC = \CC_H$ with probability at least $0.5 + \delta$. In the following lemma, we claim that event $\EE$ holds with high probability. 

\begin{restatable}{lemma}{LBGoodDist}
\label{lem:LB_goodDist}
Suppose $X$ is generated from class $\CC = \CC_H$ (similarly $\CC = \CC_L$). Then, with probability at least $1-\delta/2$, $p$ is an $(\alpha, \rho)$-scattered-heavy-tailed (similarly light-tailed) distribution, and $X$ contains at least $m$ samples. 
\end{restatable}

We prove this lemma in Section~\ref{appendix:subsec-proofGoodDist}. Using the above lemma, we can show that if $X$ is generated from a random $\CC$, then $\AA$ still distinguishes whether $\CC = \CC_L$ or $\CC = \CC_H$ with some reasonable probability while $\EE$ does not necessary holds. In particular, we have:
\begin{equation}\label{eq:distinction_prob}
\begin{split}
    \Pr{\AA(X) = \CC} & \geq \Pr{\AA(X) = \CC \mid \EE} \cdot \Pr{\EE}
    \\
    & \geq \left(\frac{1}{2} + \delta \right)\cdot \left(1-\frac{\delta}{2}\right) > \frac{1}{2} + \frac{\delta}{4}
\end{split}
\end{equation}

Now, by Le Cam's lemma, the probability of distinction is bounded by the total variation distance between the input distributions, i.e., $\DD_L$ and $\DD_H$. More formally, we have:
\begin{equation} \label{eq:distinction_leCam}
    \Pr{\AA(X) = \CC} \leq \frac{1}{2} + \frac{\|\DD_L - \DD_H\|_{tv}}{2}\,.
\end{equation}
    
Now, putting \Cref{eq:distinction_leCam} and \Cref{eq:distinction_prob} together, we obtain: 
\begin{equation} 
    \|\DD_L - \DD_H\|_{tv} > \frac{\delta}{2}\,.
\end{equation}
However, we have the following upper bound on the total variation distance between $\DD_H$ and $\DD_X$.
\begin{restatable}{lemma}{LBTVBound}
\label{lem:LB_tvBound}
For $s = \Omega(m^2)$, the total variation distance between the distribution over the sample sets, $\DD_H$ and $\DD_L$ is at most~$\delta/2$.
\end{restatable}
For the full proof of this lemma, see Section~\ref{appendix:subsec-proofTVbd}.
Thus, by having the contradicting bounds for the total variation distance, we conclude that our original hypothesis regarding the existence of $\AA$ was false. 
\end{proof}

\subsection{Proof of Lemma~\ref{lem:LB_goodDist}}
\label{appendix:subsec-proofGoodDist}
\LBGoodDist*
\begin{proof}\
We start off by showing that the sample set $X$ contains at least $m$ samples. Recall that the number of samples, $|X|$ is a Poisson random variable with mean (and variance) $m'$. We set $m' = m + 2/\delta + 2\sqrt{\delta \, m + 1}/\delta$ which satisfies $m' - m = 2\sqrt{m/\delta}$. Note that since we assumed $1/\delta$ is a constant, $m$ is $\Theta(m)$. Now, by Chebyshev's inequality, we have: 
\begin{align*}
    \Pr{|X| < m} & \leq \Pr{m' - |X| > m' - m =  \frac{2\,\sqrt{m'}}{\sqrt{\delta}}} \leq \frac{\delta}{4} \,.
\end{align*}

Clearly, when $\CC = \CC_L\,,$ $p$ is a light-tailed distribution. Now, we need show that if $p$ is drawn randomly from $\CC_H$, then it is $(\alpha, \rho)$-scattered-heavy-tailed. Our main claim concerns the mass in the fooling region: 

\begin{restatable}{lemma}{LemMassAlphaCorrupt}
\label{lem:MassAlphaCorrupt}
For a sufficiently large $s = \Omega(\log \delta^{-1}/\rho)$, if we alter chunk $i$ for $i \leq 3s/4$, then the probability mass in the \emph{fooling region} is at least $2/(3s)$ and the hazard rate of the distribution in the fooling region is decreasing by rate at least $\alpha = 0.0043$. 
\end{restatable}

We prove this lemma in Section~\ref{appendix:subsec-proofAlphaMass}. Now, let $s' > s/2$ denote the number of chunks such that $i \leq 3s/4$. Let $s_a$ denote the number of chunks we alter and replace $f$ by $f_H$ among these $s'$ intervals. It is not hard to see that $\E{s_a} = \rho' \cdot s' / 2$. Given the above lemma, if $s_a \geq 3\,s\,\rho/2$, it is clear that at least on $\rho$-fraction of the domain, the hazard rate decreases by rate at lest $\alpha$. Therefore, $p$ will be an $(\alpha, \rho)$-scattered-heavy-tailed. We set $\rho' = 12 \cdot \rho$. Now, by the Chernoff bound, we show that $s_a$ is large enough with high probability:
\begin{align*}
    \Pr{s_a < 3\,s\,\rho/2} = \Pr{\frac{s_a}{s'} < \frac{\rho'}{2}\cdot \left(1 - \frac{1}{2}\right)} \leq \exp\left(-3\, s\,\rho/8\right) \leq \delta/4.
\end{align*}
where the last inequality holds for $s = \Omega(\log \delta^{-1}/\rho)$.
Now, by the union bound, with probability $\delta/2$, we have at least $m$ samples, and a random $p \in C$ is either a light tailed distribution, or an $(\alpha, \rho)$-scattered-heavy-tailed one. 
\end{proof}

\subsection{Proof of Lemma~\ref{lem:LB_tvBound}}
\label{appendix:subsec-proofTVbd}
\LBTVBound*
\begin{proof}\
To prove the lemma, we propose a new process for drawing samples from a random $\CC \in \{\CC_L, \CC_H\}$ based on piossonization method. In this new process, instead of drawing sample from the distribution directly, we first determine the number of samples in a chunk, and then draw samples from each chunk separately.

Let $p$ be a random distribution in $\CC$. We use $m_i$ to  denote the number of sample in chunk $i$ when we draw $\poi(m')$ samples from $p$. Given the properties of the poissonization method, $m_i$ is a random variable drawn from $\poi(m'/s)$. We claim if $m_i = 1$, regardless of $\CC$ being equal to $\CC_L$ or $\CC_H$, the sample from chunk $i$, $x_i$, is drawn from the exponential distribution $\fE$ over chunk $i$. The claim is trivial when $\CC = \CC_L$ or when we have not alter chunk $i$ in $p$. Now, assume $\CC = \CC_H$, and the alteration happens. In this case, since $p$ is randomly selected, the conditional distribution in chunk $i$ is either $\fH$ or $\fbar$ each with probability half. Recall that we define $\fbar$ to be in such a way that the mixture of $\fH$ and $\fbar$ is exactly $\fE$. Thus, for a random $p  \in \CC_H$, when $m_i = 1$, sample $x$ comes from exactly $\fE$. 

The above property implies that a sample set $X$ for which each chunk has one sample will be generated with the same probability regardless of the choice of $\CC$. Therefore, the total variation distance between $\DD_H$ and $\DD_L$ is bounded by the probability of at seeing at least two sample from a chunk. By properties of the Poisson distribution, we show it is very unlikely to see two samples from the same chunk:

\begin{align*}
    \Pr{m_i \geq 2} = 1 - \Pr{m_i = 0} - \Pr{m_i = 1} = 1 - e^{-m'/s} - m' \, e^{-m'/s}/s \leq \frac{m'^2}{2\,s^2}\,.
\end{align*}
Using the union bound, the probability of having a two samples from the same chunk is 
\begin{align*}
    \|\DD_L - \DD_H\|_{tv} \leq \Pr{\exists i : m_i \geq 2} \leq \frac{m'^2}{2\,s} \leq \frac{\delta}{2}\,.
    \end{align*}
where the last inequality is true when $s \geq {m'}^2/\delta = \Theta(m^2)$.
\end{proof}

\subsection{Proof of Lemma~\ref{lem:MassAlphaCorrupt}}
\label{appendix:subsec-proofAlphaMass}
\LemMassAlphaCorrupt*

\begin{proof}
We prove this lemma in two parts. First, we show that the derivative of the hazard rate is bounded above in the regions in which it is decreasing. Then, we show that in those decreasing regions, the probability mass is a constant fraction of the total mass in that chunk. We previously set $\beta = 1.5\,$ which we will use here, as well.

\paragraph{Bounded Derivative of Hazard Rate} First, we calculate the hazard rate for a chunk:

\begin{align}
HR_\text{hard}(x) &= \begin{cases}
\frac{Ae^{-\beta (x-x_A)}}{1 - \left(F_\text{exp}(x_A) + \frac{A}{\beta}\left(1 - e^{-\beta(x-x_A)}\right)\right)} &  x \in [x_A, x_1) \\
\frac{h}{1 - \left( F_\text{exp}(x_A) + \frac{A}{\beta}\left(1 - e^{-\beta(x_1-x_A)}\right) + h(x - x_1) \right)}  & x \in [x_1, x_2) \\
\frac{Be^{-\beta (x - x_B)} }{1 - \left( F_\text{exp}(x_A) + \frac{A}{\beta}\left(1 - e^{-\beta(x_1-x_A)}\right) + h(x_2 - x_1) + \frac{B}{\beta}(e^{-\beta (x-x_B)} - 1)   \right)} & x \in [x_2, x_B)
\end{cases} \\
&= \begin{cases}
\frac{A}{\left(A - \frac{A}{\beta}\right) e^{\beta(x-x_A)} + \frac{A}{\beta}} &  x \in [x_A, x_1) \\
\frac{h}{A - \frac{A}{\beta}\left( 1 - e^{-\beta(x_1 - x_A)}\right) - h(x-x_1)}  & x \in [x_1, x_2) \\
\frac{B}{\left(A - \frac{A}{\beta} \left( 1- e^{-\beta(x_1-x_A)}\right) - h (x_2 - x_1) + \frac{B}{\beta} \right) e^{\beta(x-x_B)} + \frac{B}{\beta}} & x \in [x_2, x_B)
\end{cases} 
\end{align}
We consider the derivative of the hazard rate to determine the fooling regions and show that the derivative is below $-\alpha$:

\begin{align}
\frac{d}{dx} HR(x) = \begin{cases}
\frac{-(\beta - 1) e^{\beta(x-x_A)}}{\left(  \left(1-\frac{1}{\beta}\right) e^{\beta (x-x_A)} + \frac{1}{\beta}   \right)^2} & x \in [x_A, x_1) \\
\frac{h^2}{\left(  C_1 - hx\right)^2   } & x \in (x_1, x_2) \\
\frac{- B \beta C_2 e^{\beta(x-x_B)}}{  \left(  C_2 e^{\beta(x-x_B)} + \frac{B}{\beta}    \right)^2  } & x \in (x_2, x_B)
\end{cases}
\end{align}
where $C_1 = \left(   A - \frac{A}{\beta}\left( 1 - e^{-\beta(x_1-x_A)}\right) +hx_1\right)$ and $C_2 = \left(A - \frac{A}{\beta} \left( 1- e^{-\beta(x_1-x_A)}\right) - h (x_2 - x_1) + \frac{B}{\beta} \right) \,.$ We first observe that because $\beta > 1\,,$ $C_1, C_2 > 0\,,$ the derivative of the hazard rate is positive in the uniform region and negative in the two exponential regions. Thus, these are the fooling regions. We now show that for a constant $\alpha$, the derivatives of the hazard rate in the fooling regions are at most $-\alpha\,.$ First, we lower bound the absolute value of this quantity in the first exponential region:

\begin{align}
\frac{(\beta - 1) e^{\beta(x-x_A)}}{\left(  \left(1-\frac{1}{\beta}\right) e^{\beta (x-x_A)} + \frac{1}{\beta}   \right)^2} &\ge \frac{\beta - 1}{\left(  \left(1-\frac{1}{\beta}\right) e^{\beta (x-x_A)} + \frac{1}{\beta}   \right)^2} \\
&\ge \frac{\beta -1}{\left(  \left( 1-\frac{1}{\beta} \right) \frac{A}{h} + \frac{1}{\beta} \right)^2} = \frac{\beta^2 (\beta - 1)}{\left( (\beta - 1)\frac{A}{h} + 1\right)^2} \\
&\ge \frac{\beta^2 (\beta - 1)}{\left(\frac{\beta - 1}{h} + 1 \right)^2} \label{eqn:bd-hr-firstexp}
\end{align}

Next, we lower bound the absolute value of the derivative of the hazard rate in the second exponential region:
\begin{align}
\frac{B \beta C_2 e^{\beta(x-x_B)}}{\left( C_2 e^{\beta(x-x_B)} + \frac{B}{\beta}   \right)^2} &\ge \frac{B \beta C_2 B/h}{\left( C_2 + \frac{B}{\beta} \right)^2} = \frac{\beta C_2 B^2}{h\left(C_2 + \frac{B}{\beta}\right)^2} \label{eqn:in-c2}\\
&\ge \frac{\beta B^3}{h\left( A + \frac{B}{\beta} + \frac{B}{\beta} \right)^2} \ge \frac{\beta B^3}{\left(1+ \frac{2}{\beta }  \right)^2} \,. \label{eqn:bd-hr-secondexp}
\end{align}
The transition from Eqn.~\ref{eqn:in-c2} to Eqn.~\ref{eqn:bd-hr-secondexp} results from bounding $C_2$ as follows:

\begin{align}
C &= A - B = \frac{A}{\beta} \left( 1 - e^{-\beta(x_1 -x_A)}\right) + h(x_2 - x_1) + \frac{B}{\beta} \left(  e^{-\beta(x_2 - x_B)} - 1 \right) \\
&\Rightarrow A -  \frac{A}{\beta} \left( 1 - e^{-\beta(x_1 -x_A)}\right)  - h(x_2 - x_1) + \frac{B}{\beta} = B + \frac{B}{\beta} e^{-\beta(x_2 - x_B)} = C_2\,.
\end{align}
Thus, $A + \frac{B}{\beta} > C_2 > B\,.$ \par

Finally, it suffices to bound these quantities for intervals that lie in a certain region of the distribution that comprises a constant fraction of the mass of the distribution. Thus, since we have that $i \leq 3s/4$, we know that $x_B \le -\ln(1/4) = \ln(4)$, implying that $1/4\leq A, B, h$.This allows us to bound Eqn.~\ref{eqn:bd-hr-secondexp} by:
$$
\frac{\beta B^3}{\left(1+ \frac{2}{\beta }  \right)^2} \ge \frac{\beta (1/4)^3}{\left(1+ \frac{2}{\beta }  \right)^2} = \frac{\beta}{4^3\left(1+ \frac{2}{\beta }  \right)^2}
$$
We also require a lower bound on $h\,$; by definition, we have that $h \ge B = e^{-x_B} \ge 1/4\,.$
Finally, we plug this back into the bounds from Eqn.~\ref{eqn:bd-hr-firstexp} and Eqn.~\ref{eqn:bd-hr-secondexp} to determine $\alpha\,:$

\begin{align*}
\alpha &= \min\left(  \frac{\beta^2 (\beta - 1)}{\left(\frac{\beta - 1}{1/4} + 1 \right)^2} , \frac{\beta}{4^3\left(1+ \frac{2}{\beta }  \right)^2} \right) \\
&= \min\left(  \frac{\beta^2 (\beta - 1)}{\left(4 \beta + 3 \right)^2} ,    \frac{\beta^3 }{4^3\left(\beta+ 2  \right)^2}\right) 
\end{align*}

When $\beta = 1.5\,, \alpha \approx 0.0043 \,.$

\paragraph{Area Under Fooling Region} Finally, we must show that the area under the fooling region is large. To do this, we evaluate $F_\text{hard}(x_B) - F_\text{hard}(x_2) + F_\text{hard}(x_1) - F_\text{hard}(x_A)$ and show that it is a large fraction of $C$ for the whole domain.  \par

Evaluating this area, we get:
\begin{align}
F_\text{hard}(x_B) - F_\text{hard}(x_2) + F_\text{hard}(x_1) - F_\text{hard}(x_A) &= \frac{A}{\beta}\left(1 - e^{-\beta (1/\beta \ln(A/h))} \right) + \frac{B}{\beta} \left( e^{-\beta (1/\beta \ln(B/h))} - 1\right) \\
&= \frac{A}{\beta} - \frac{h}{\beta} + \frac{h}{\beta} - \frac{B}{\beta} = \frac{C}{\beta}
\end{align}
Over the whole domain of the distribution, this is a constant fraction of $C$, the total area in a chunk. When $\beta = 1.5\,,$ as before, $1/1.5 = 2/3$ of the area in the chunk comes from the fooling region. \par

Thus, we have shown that in a chunk, a constant fraction of the area lies in a region with hazard rate $\le -\alpha\,,$ for a constant $\alpha\,.$

\end{proof}

\section{Discussion of Experiments} \label{appendix:discussion_of_experiments}
Here, we include a version of the algorithm with weaker theoretical guarantees that we used in our experiments. We also discuss some preliminary experiments run on real-world data. Finally, we present the details of the experiments run to compare our algorithm to a naive one.

 \begin{algorithm}
 \caption{Weak $(\alpha, \rho)$-Heavy-Tailed Test}  \label{alg:weak_tester}
 \begin{algorithmic}[1]
 \State $\mathcal{S}_1, \mathcal{S}_2, \mathcal{S}_3, \mathcal{S}_4 \leftarrow$ each $n$ samples from the distribution  
 \State Sort $\mathcal{S}_1, \mathcal{S}_2, \mathcal{S}_3, \mathcal{S}_4$ 
 \State Split into $k$ equal weight buckets and determine interval endpoints $I$ 
 \State Calculate $L[i] = I[i + 1] - I[i]$ and $dL[i] = L[i+1] - L[i]$ 
 \State Calculate $S[i] = \frac{L[i]}{dL[i]}$ for $i \in \{1, 2, ..., k-2\}$ 
 \If{$S[i] < 1 - \frac{i}{k} - \frac{1}{2}\text{gap}(\alpha)$ for any $i \in \{c_1 \cdot k, ..., c_2 \cdot k\} $ }
 \State PASS
 \Else
 \State FAIL
 \EndIf
 \end{algorithmic}
 \end{algorithm}

%
%
%
%
%
%
%

\subsection{TPC-H Job Distribution}
\paragraph{Data and Methods} We run experiments on a 2.8 GHz Intel i7 core using the algorithm presented in Appendix~\ref{appendix:discussion_of_experiments}. The dataset considered is sampled from TPC-H queries (\cite{noauthor_tpc-h_nodate}). We use the dataset curated by~\cite{mao_learning_2019}, in which they sample jobs from different input sizes from 22 different TPC-H queries. Our experiments verify their claim that the distribution of jobs over work is heavy-tailed, which they justify by noting that about $20\%$ of the jobs contain about $80\%$ of the work, a common heuristic for judging heavy-tailedness. In Figure~\ref{fig:tpc-h_expt_results}, we plot the distribution of job durations from the dataset and overlay an approximate Lomax fit. Notably, the distribution is not monotone, which means it does not entirely obey our assumptions. Further, due to a limited number of jobs, the distribution does not lie on an infinite domain. However, we show that the algorithm still picks up on heavy-tailedness. We sample $n = \Theta(k^4) \approx 56$ million samples for $k = 50$. 
\paragraph{Results and Discussion} We find that the test statistic considers the distribution heavy-tailed, since not all of the indices considered ``look'' light-tailed in our statistic. In particular, the red dashed line represents the test statistic calculated on samples from the distribution of jobs over work in TPC-H. Thus, we are able to verify the claim of the authors of~\cite{mao_learning_2019} that this distribution is heavy-tailed. It is worth noting that we have applied our algorithm to a practical setting where several of the assumptions of our work do not hold, and the algorithm still shows some signal when consider a heavy-tailed distribution. In order for this validation to be complete, however, we would need to show that a light-tailed distribution that doesn't match our assumptions exactly does not behave unexpectedly. 

\begin{figure}[H]
    \centering
    \begin{tabular}{{p{0.5\textwidth}p{0.5\textwidth}}}
     \includegraphics[width=0.5\textwidth]{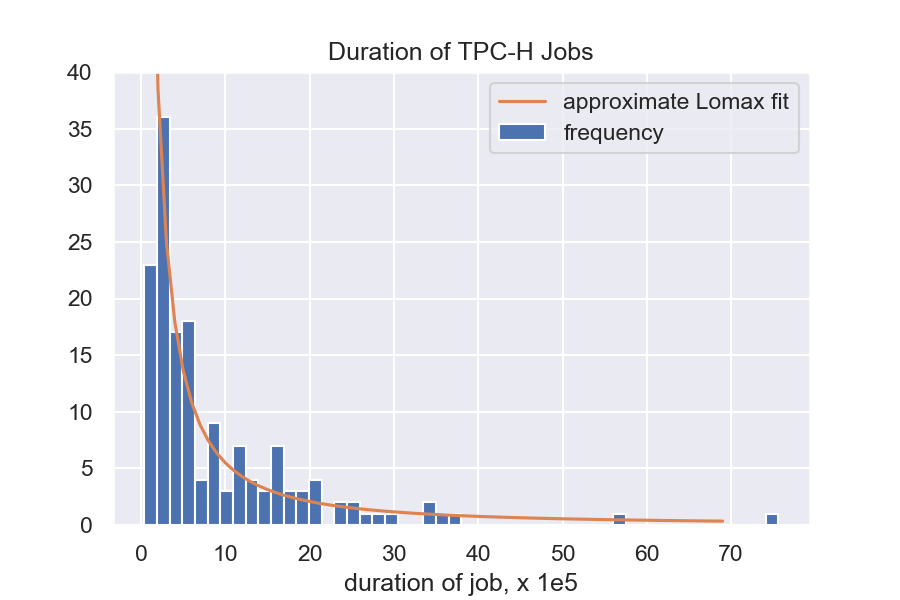} &          \includegraphics[width=0.5\textwidth]{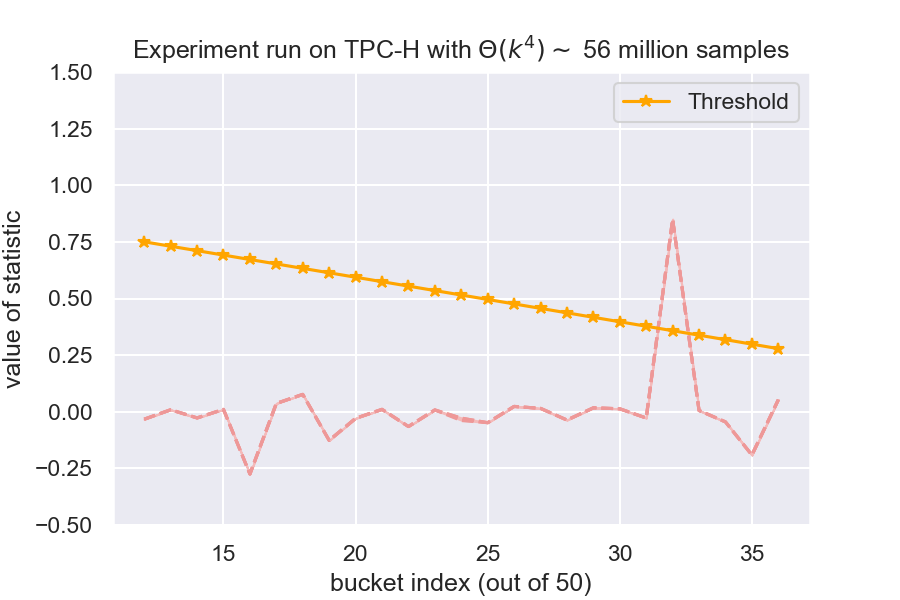} \\
         (a) & (b) \\
    \end{tabular}
    
    \caption{The plot on the left shows the distribution of job duration with an approximate Lomax fit. On the right, in (b), we plot the calculated value of the test statistic in comparison to the threshold for considering a distribution heavy-tailed. Since the dashed red line is below the orange at at least one point, the algorithm would consider this distribution heavy-tailed, in line with the assessment of the authors of~\cite{mao_learning_2019}. We note that we don't see a significant spread over many runs, since the underlying distribution is finite. For the same reason, increasing the number of samples would not add any benefit.}
    \label{fig:tpc-h_expt_results}
\end{figure}

\subsection{Comparison Against a Naive Algorithm}
\label{appendix:subsec:dkw-comparison}
In this section, we describe experiments we ran to compare our algorithm to a naive algorithm, showing that our algorithm has an edge in distinguishing difficult cases. We first explain the experimental setup, including some context for why we chose the examples we did and the naive algorithm. We then describe the results we saw, and finally posit an explanation as to why this is the case.

\paragraph{Motivation and Experimental Setup} Our algorithm, to our knowledge, is the first finite-sample algorithm for this problem over a continuous domain. As simple as our algorithm is, a natural question is that of how a very naive algorithm might perform. In particular, we could attempt to learn the CDF from samples, and then directly calculate the derivative of the hazard rate from that. Based on the Dvoretzky–Kiefer–Wolfowitz (DKW) inequality, we know that with enough samples, we can get a good approximation of the CDF. Thus, a simple idea would be to approximate the PDF and derivative of the PDF from the empirical CDF and from there compute an approximation to the derivative of the hazard rate. Since the DKW inequality provides a guarantee in terms of the supremum of the difference between the true and empirical CDFs, a hard case for an algorithm that learns the CDF would be one where the CDFs of two distributions are very close to one another in CDF but are classified differently, one as heavy-tailed and one as light-tailed. To this end, we choose the likelihood $f(x) = \exp(-x/\ln(x))$\footnote{the normalization constant is hard to calculate analytically so we did so numerically.} and the PDF for an exponential, $f(x) = \exp(-x)\,.$ \par
As discussed in \Cref{appendix:discussion_of_experiments}, we implemented a weaker version of our algorithm that didn't involve the finer-grained buckets. We sought to distinguish the aforementioned distributions based on samples drawn from the likelihood  Fixing the number of samples, over the course of 50 repetitions, we set half of them to be exponential and half to be the hard case, sampled from the chosen one, and ran both our tester and the naive tester. We used appropriate settings for the constants $\beta, B_1, B_2$ for the threshold of our algorithm. We then computed what percentage of the time each algorithm had been correct, and we conducted this process for various sample sizes.

\paragraph{Results and Discussion} From analytically calculating the hazard rate, we see that the distribution has decreasing hazard rate over most of its support. However, from the plot of the hazard rate from samples, the naive algorithm considers the distribution light-tailed early on, and almost exponential later (constant hazard rate) (Figure~\ref{fig:hard_Case_hr}). Our algorithm, too, confuses the distribution for exponential late in the distribution. However, with the same number of samples, crucially, it identifies the heavy-tailed-ness early on (Figure~\ref{fig:our_alg_on_hard}).

\begin{figure}[H]
     \centering
    \begin{tabular}{{p{0.5\textwidth}p{0.5\textwidth}}}
     \includegraphics[width=0.5\textwidth]{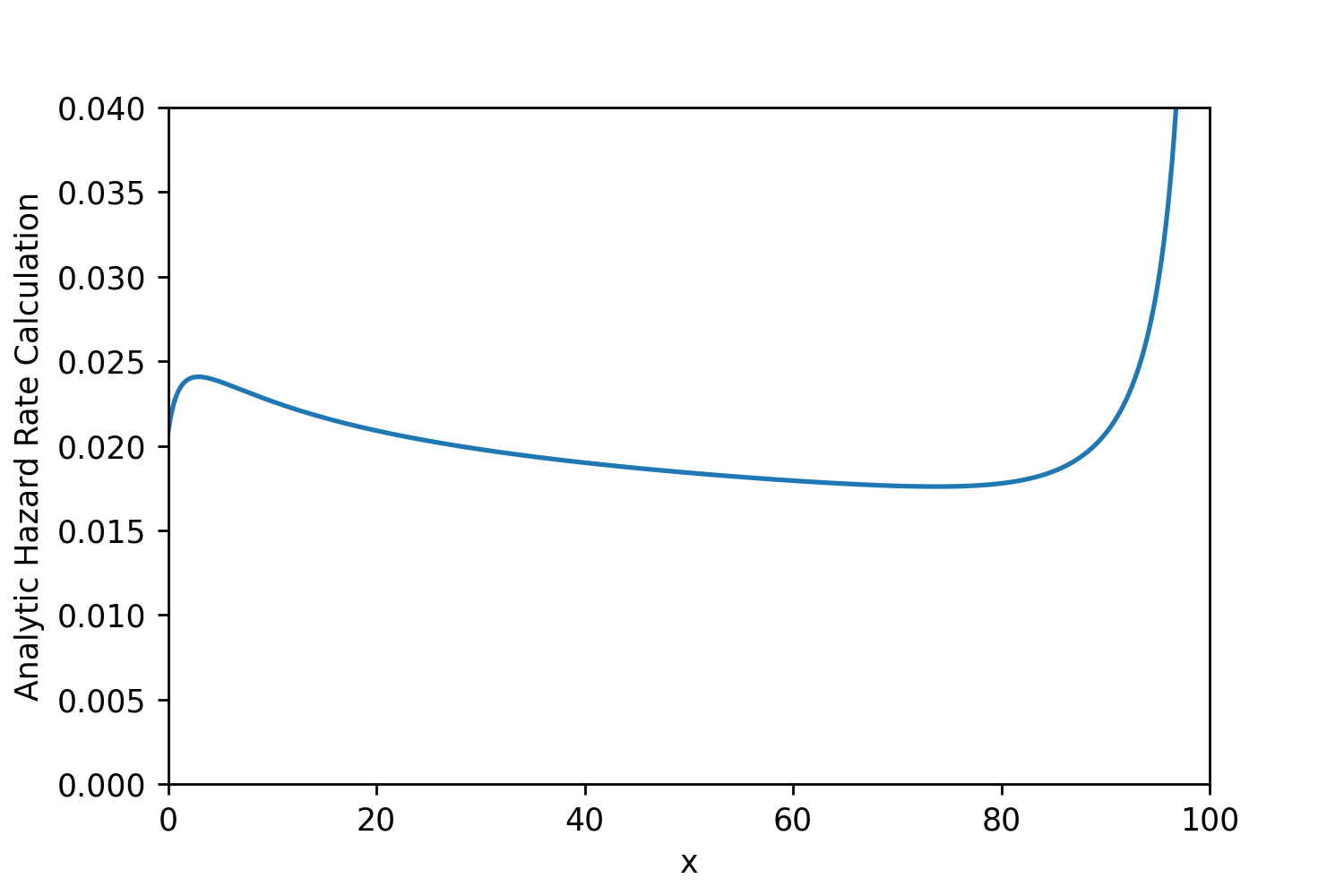} &          \includegraphics[width=0.5\textwidth]{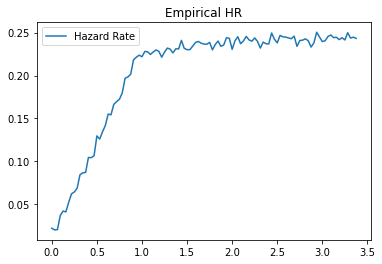} \\
         (a) & (b) \\
    \end{tabular}

    \caption{Analytic Hazard Rate and Hazard Rate calculated by the naive algorithm for a hard distribution, $f(x) = \exp(-x/\log(x))$}
    \label{fig:hard_Case_hr}
\end{figure}

\begin{figure}
    \centering
    \includegraphics[width=0.6\textwidth]{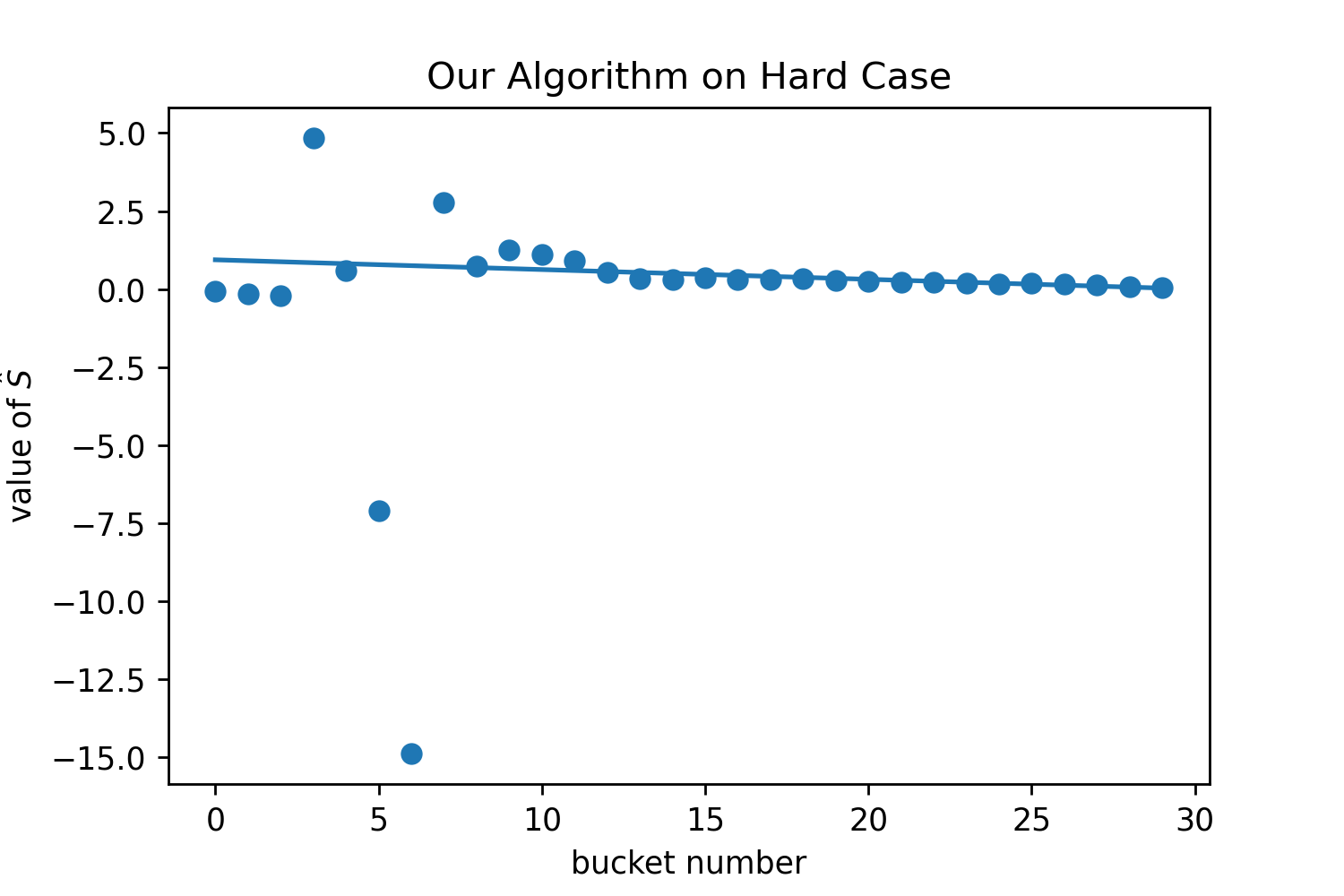}
    \caption{The value of the statistic for our algorithm (averaged over two runs) also confuses the distribution with an exponential toward the end of the domain (as evidenced by the coincidence of the statistic with the threshold) but determines the heavy-tailed-ness from early buckets.}
    \label{fig:our_alg_on_hard}
\end{figure}

Finally, we consider the success rates of our algorithm vs the naive one. In Figure~\ref{fig:acc_resp_hard}, we note that our algorithm starts to do better than random guessing with around $10^7$ samples, where the naive algorithm still does not do better than random guessing.

\begin{figure}
    \centering
    \includegraphics[width=0.6\textwidth]{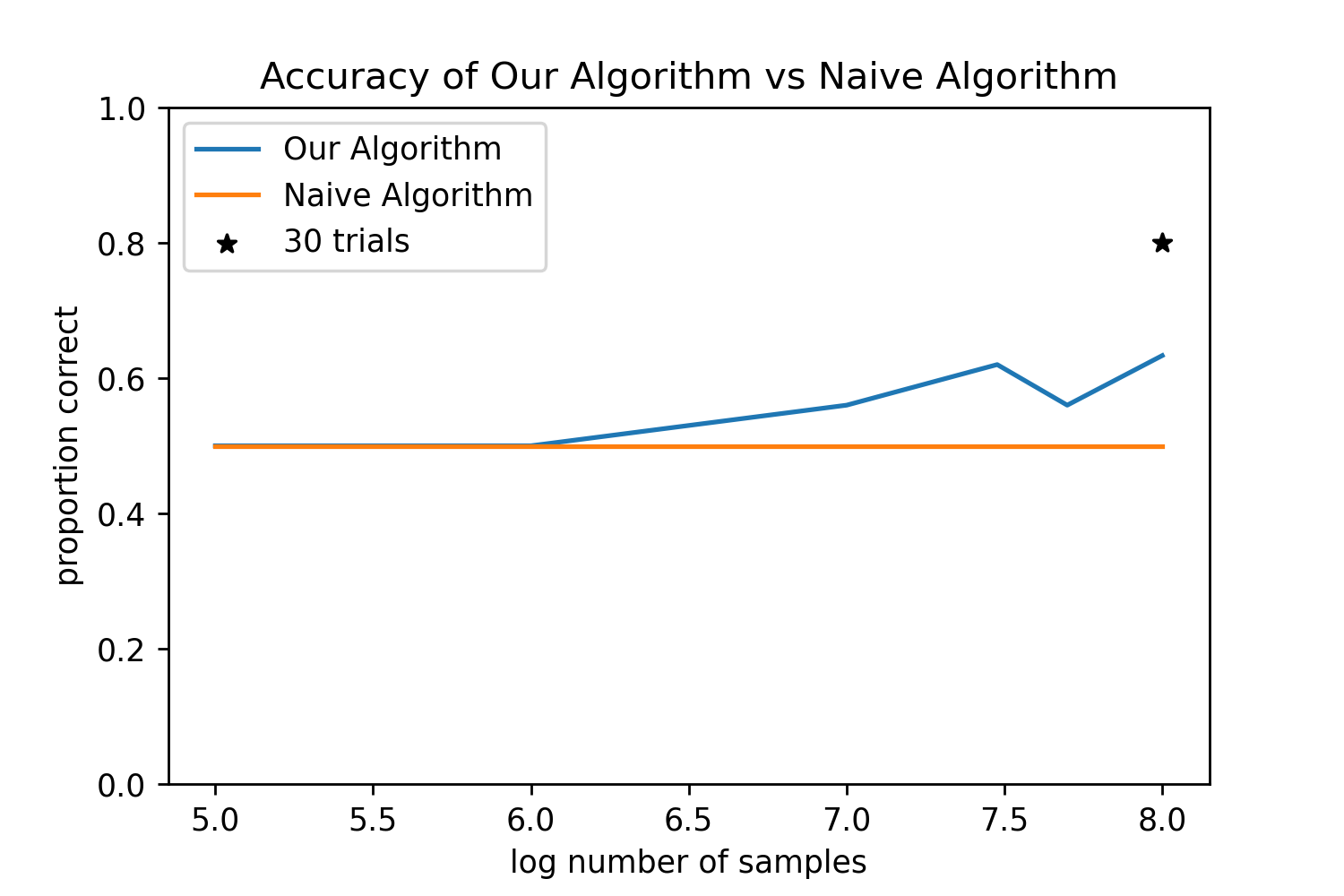}
    \caption{Performance of Each Algorithm. Note that 0.5 is random guessing, and so only our algorithm does better than that at any point.}
    \label{fig:acc_resp_hard}
\end{figure}

\clearpage

\section{Low-sample Analysis}
\label{appendix:sec-low-sample}
We analyze the performance of our algorithm when it is given fewer samples than what our theoretical analysis requires. In specific applications, where the provable guarantees we give may not be required, we may be able to get good enough results without as many buckets and samples.

\begin{table}[H]
    \centering
    \begin{tabular}{c||c|c|c}
    	&  \multicolumn{3}{r}{$\rightarrow$ more samples per bucket}   \\
	\hline & & & \\
	5 buckets & \includegraphics[width=0.2 \textwidth]{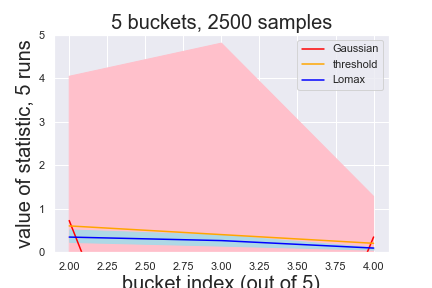}& \includegraphics[width=0.2 \textwidth]{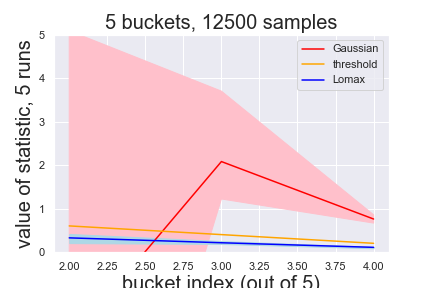}&\includegraphics[width=0.2 \textwidth]{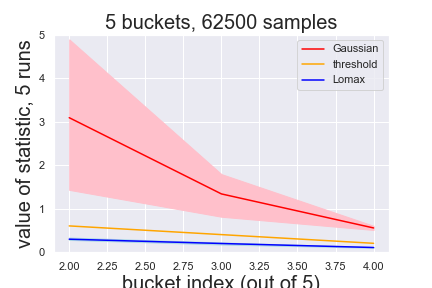} \\

         10 buckets & \includegraphics[width=0.2 \textwidth]{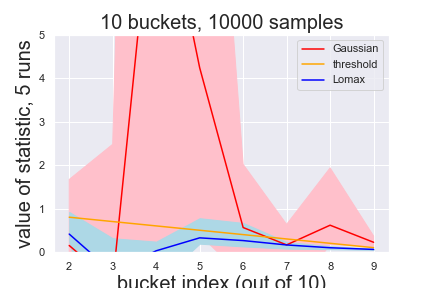}& \includegraphics[width=0.2 \textwidth]{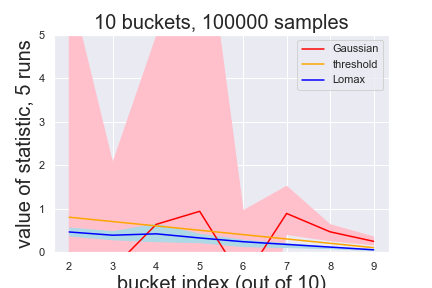}&\includegraphics[width=0.2 \textwidth]{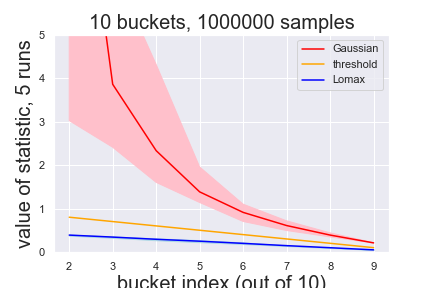} \\
         20 buckets & \includegraphics[width=0.2 \textwidth]{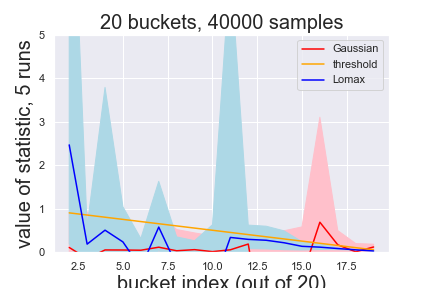} &\includegraphics[width=0.2 \textwidth]{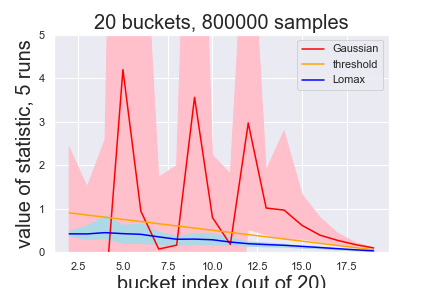} & \includegraphics[width=0.2 \textwidth]{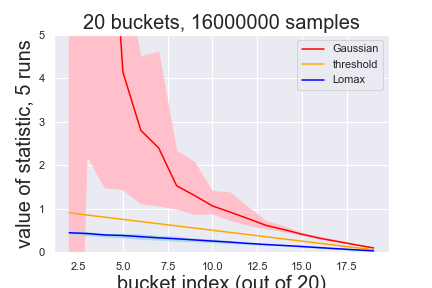} \\
         30 buckets &  \includegraphics[width=0.2 \textwidth]{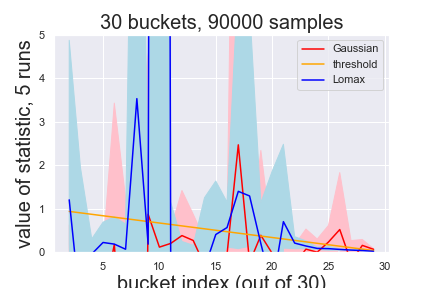} &\includegraphics[width=0.2 \textwidth]{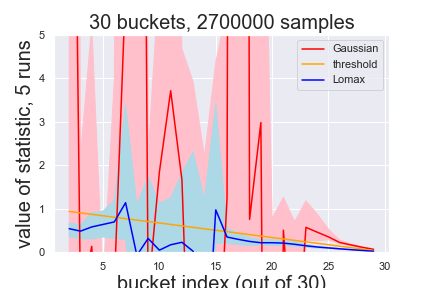} & \includegraphics[width=0.2 \textwidth]{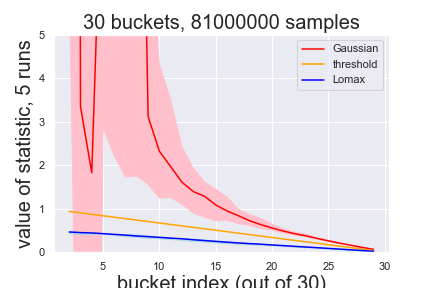}\\
    \end{tabular}
    \caption{ In this table, we consider (half) Gaussian (light-tailed) and Lomax (heavy-tailed) distributions. We use smaller sample sizes (2500 to 20 million) than the experiments in the main paper (Except for the bottom-rightmost plot, which has 81 million samples).
Observe that with fewer samples per bucket (further to left), there is much more variation (pink area and light blue area) in the computed statistics, and they do not behave as we would wish for our statistic to behave. In the rightmost column, however, we start to see the value of our statistic gets very close to what we theoretically expect, even with a small value for $k$. In particular, it concentrates well in the tail. Indeed, in the rightmost column we see that for smaller values for $k\,$ (62,500 samples for $k = 5$) the test clearly distinguishes between the (half) Gaussian and Lomax distributions.
Despite promising behavior of the test even from such few buckets and samples, it is essential to note that large $k$ is theoretically necessary to capture the behavior of an \textit{arbitrary} underlying (well-behaved) distribution. However, in specific applications, these provable guarantees may not be required, and good enough results may be seen without as many buckets and samples in these cases.}   
 \label{tab:my_label}
\end{table}

\end{document}